\documentclass[preprint,12pt]{elsarticle}

\usepackage{amssymb}
\usepackage{amsmath}
\usepackage{amsthm}

\usepackage{lineno}
\journal{Artificial Intelligence Journal}

\usepackage{booktabs,makecell,multirow}
\usepackage{placeins}
\usepackage[hidelinks]{hyperref}
\usepackage{amsxtra, amsfonts, amssymb, amstext, amsmath,mathtools}%, 
\usepackage{xspace}
\usepackage[algo2e,ruled,vlined,linesnumbered]{algorithm2e}
    \SetKwInOut{Input}{Input}
    \SetKwInOut{Output}{Output}
    \ResetInOut{output}
    \SetKw{Break}{break}
    \SetKw{Breakall}{break all}

 \newcommand\blfootnote[1]{%
   \begingroup
   \renewcommand\thefootnote{}\footnote{#1}%
   \addtocounter{footnote}{-1}%
   \endgroup
 }

\newtheorem{theorem}{Theorem}
\newtheorem{lemma}[theorem]{Lemma}

\newcommand{\NSGAtwo}{NSGA\nobreakdash-II\xspace}
\newcommand{\NSGAthree}{NSGA\nobreakdash-III\xspace}
\newcommand{\SMS}{SMS\nobreakdash-EMOA\xspace}
\newcommand{\SPEA}{SPEA2\xspace}
\newcommand{\SIBEA}{$(\mu+1)$-SIBEA\xspace}

\newcommand{\onemax}{\textsc{OneMax}\xspace}

\newcommand{\LO}{\textsc{Leading\-Ones}\xspace}
\newcommand{\leadingones}{\LO}
\newcommand{\TZ}{\textsc{Trailing\-Zeros}\xspace}

\newcommand{\jump}{\textsc{Jump}\xspace}
\newcommand{\zerojump}{\textsc{Zero\-Jump}\xspace}
\newcommand{\omm}{\textsc{OMM}\xspace}
\newcommand{\xOMM}[1]{\text{$#1$\textsc{OMM}}\xspace}

\newcommand{\mOMM}{\xOMM{m}}
\newcommand{\mLOTZ}{\text{$m$\textsc{LOTZ}}\xspace}
\newcommand{\cocz}{\textsc{COCZ}\xspace}
\newcommand{\lotz}{\textsc{LOTZ}\xspace}
\newcommand{\ojzj}{\textsc{OJZJ}\xspace}
\newcommand{\mlotz}{$m$\textsc{LOTZ}\xspace}
\newcommand{\mcocz}{$m$\textsc{COCZ}\xspace}
\newcommand{\mJump}{\text{$m\text{\textsc{OJZJ}}_k$}\xspace}

\DeclareMathOperator{\mutate}{mutate}

\newcommand{\R}{\ensuremath{\mathbb{R}}}
\newcommand{\Q}{\ensuremath{\mathbb{Q}}}
\newcommand{\N}{\ensuremath{\mathbb{N}}} % ohne Null!!!
\newcommand{\Z}{\ensuremath{\mathbb{Z}}}

\newcommand{\calA}{\ensuremath{\mathcal{A}}} 
\newcommand{\calB}{\ensuremath{\mathcal{B}}}

\newcommand{\fmax}{f_{\mathrm{max}}}

\let\originalleft\left
\let\originalright\right
\renewcommand{\left}{\mathopen{}\mathclose\bgroup\originalleft}
\renewcommand{\right}{\aftergroup\egroup\originalright}

\newcommand{\nootherflip}[1][n-1]{\left(1-\frac{1}{n}\right)^{#1}}

\DeclarePairedDelimiter\floor{\lfloor}{\rfloor}
\DeclarePairedDelimiter\ceil{\lceil}{\rceil}
\newcommand{\abs}[1]{\left| #1\right|\xspace}

\let\oldsqrt\sqrt
\def\hksqrt{\mathpalette\DHLhksqrt}
\def\DHLhksqrt#1#2{\setbox0=\hbox{$#1\oldsqrt{#2\,}$}\dimen0=\ht0
   \advance\dimen0-0.2\ht0
   \setbox2=\hbox{\vrule height\ht0 depth -\dimen0}%
   {\box0\lower0.4pt\box2}}
\renewcommand\sqrt\hksqrt

\DeclareMathOperator{\LOm}{LO}
\DeclareMathOperator{\TZm}{TZ}
\newcommand{\corners}{\ensuremath{C_m}\xspace}

\newcommand{\cliffs}{\ensuremath{K_{m,k}}\xspace}
\newcommand{\cornersjump}{\ensuremath{C_{m,k}}\xspace}

\newcommand{\set}[1]{\{#1\}}

\newcommand{\gcstring}[3]{$(#1,#2,\ldots,#3)$-bitstring\xspace}

\newcommand{\gabstring}{\gcstring{a_1}{a_2}{a_{m'}}}

\DeclareMathOperator{\maOMM}{OMM}
\DeclareMathOperator{\maCOCZ}{COCZ}
\DeclareMathOperator{\maLOTZ}{LOTZ}
\DeclareMathOperator{\maOJZJ}{OJZJ}

\newcommand{\SOMM}{S_m^{\maOMM}}
\newcommand{\SCOCZ}{S_m^{\maCOCZ}}
\newcommand{\SLOTZ}{S_m^{\maLOTZ}}
\newcommand{\SOJZJ}{S_{m,k}^{\maOJZJ}}

\newcommand{\frontOMM}{M_m^{\maOMM}}
\newcommand{\frontCOCZ}{M_m^{\maCOCZ}}
\newcommand{\frontLOTZ}{M_m^{\maLOTZ}}
\newcommand{\frontOJZJ}{M_{m,k}^{\maOJZJ}}

\DeclareMathOperator{\half}{half}

\begin{document}
{\sloppy

\begin{frontmatter}

%% Title, authors and addresses

%% use the tnoteref command within \title for footnotes;
%% use the tnotetext command for theassociated footnote;
%% use the fnref command within \author or \affiliation for footnotes;
%% use the fntext command for theassociated footnote;
%% use the corref command within \author for corresponding author footnotes;
%% use the cortext command for theassociated footnote;
%% use the ead command for the email address,
%% and the form \ead[url] for the home page:
%% \title{Title\tnoteref{label1}}
%% \tnotetext[label1]{}
%% \author{Name\corref{cor1}\fnref{label2}}
%% \ead{email address}
%% \ead[url]{home page}
%% \fntext[label2]{}
%% \cortext[cor1]{}
%% \affiliation{organization={},
%%             addressline={},
%%             city={},
%%             postcode={},
%%             state={},
%%             country={}}
%% \fntext[label3]{}

\title{Near-Tight Runtime Guarantees for Many-Objective Evolutionary Algorithms}

%% use optional labels to link authors explicitly to addresses:
\author[label1]{Simon Wietheger}
\affiliation[label1]{organization={Algorithms and Complexity Group, TU Wien},
            % addressline={},
            city={Vienna},
            % postcode={},
            % state={},
            country={Austria}}

\author[label2]{Benjamin Doerr}
\affiliation[label2]{organization={Laboratoire d’Informatique (LIX), CNRS, École Polytechnique, Institut Polytechnique de Paris},
            % addressline={},
            city={Palaiseau},
            % postcode={},
            % state={},
            country={France}}

%% Abstract
\begin{abstract}
%% Text of abstract
Despite significant progress in the field of mathematical runtime analysis of multi-objective evolutionary algorithms (MOEAs), the performance of MOEAs on discrete many-objective problems is little understood. In particular, the few existing performance guarantees for classic MOEAs on classic benchmarks are all roughly quadratic in the size of the Pareto front.

In this work, we consider a large class of MOEAs including the (global) SEMO, SMS-EMOA, balanced NSGA-II, NSGA-III, and SPEA2. For these, we prove near-tight runtime guarantees for the four most common benchmark problems OneMinMax, CountingOnesCountingZeros, LeadingOnesTrailingZeros, and OneJumpZeroJump, and this for arbitrary numbers of objectives.
Most of our bounds depend only linearly on the size of the largest incomparable set, showing that MOEAs on these benchmarks cope much better with many objectives than what previous works suggested. Most of our bounds are tight apart from small polynomial factors in the number of objectives and length of bitstrings. This is the first time that such tight bounds are proven for many-objective uses of MOEAs. %SEMO
For the runtime of the SEMO on the LOTZ benchmark in $m \ge 6$ objectives, our runtime guarantees are even smaller than the size of the largest incomparable set. This is again the first time that such runtime guarantees are proven.
\blfootnote{An extended abstract of this work was first published in Parallel Problem Solving from Nature – PPSN XVIII, p. 153-168, 2024, by Springer Nature.}
\end{abstract}

% %%Graphical abstract
% \begin{graphicalabstract}
% %\includegraphics{grabs}
% \end{graphicalabstract}

% %%Research highlights
% \begin{highlights}
% \item Research highlight 1
% \item Research highlight 2
% \end{highlights}

%% Keywords
\begin{keyword}
%% keywords here, in the form: keyword \sep keyword

evolutionary multi-objective optimization \sep%
many-objective optimization \sep%
runtime analysis \sep%
theory

%% PACS codes here, in the form: \PACS code \sep code

%% MSC codes here, in the form: \MSC code \sep code
%% or \MSC[2008] code \sep code (2000 is the default)

\end{keyword}

\end{frontmatter}

%% Add \usepackage{lineno} before \begin{document} and uncomment 
%% following line to enable line numbers
%% \linenumbers

%% main text
%%

% \begin{linenumbers}
\section{Introduction}
\label{sec:intro}

Evolutionary algorithms such as the MOEA/D, \NSGAtwo, \SMS or \SPEA are among the most successful approaches to tackle optimization problems with several conflicting 
objectives~\cite{CoelloLV07,ZhouQLZSZ11}. 
Despite the challenges stemming from the often complex population dynamics in multi-objective evolutionary algorithms (MOEAs), the analysis of MOEAs via theoretical means has made considerable progress in the last twenty years. Starting with simplistic algorithms such as the \emph{simple evolutionary multi-objective optimizer (SEMO)}~\cite{LaumannsTZWD02} and the global SEMO~\cite{Giel03}, this line of research has now reached the maturity to deal with state-of-the-art algorithms such as the MOEA/D, \NSGAtwo, \NSGAthree, \SMS, and \SPEA~\cite{LiZZZ16,ZhengD23aij,WiethegerD23,BianZLQ23,RenBLQ24}.

However, this progress so far was mostly restricted to the analysis of MOEAs on bi-objective optimization problems. In particular, the few mathematical runtime analyses, prior to this work, for the SEMO, global SEMO, and \SMS~\cite{LaumannsTZ04,BianQT18ijcaigeneral,ZhengD24}
for problems with general number $m$ of objectives show runtime guarantees roughly quadratic in the size of the Pareto front; the recent work~\cite{ZhengD24many} proves that the \NSGAtwo cannot optimize the \textsc{OneMinMax} problem in time better than exponential in the size of the Pareto front when the number of objectives is three or more. These results could give the impression that MOEAs have significant difficulties in dealing with larger numbers of objectives, clearly beyond the mere increase of the size of the Pareto front with increasing numbers of objectives. 

In this work, we revisit this twenty years old question and prove significantly stronger performance guarantees, which give a different impression. More specifically, we analyze the runtimes of the SEMO, global SEMO, balanced \NSGAtwo, \NSGAthree, \SMS,  and \SPEA on the \textsc{OneMinMax} (\omm), \textsc{CountingOnesCountingZeros} (\cocz), \textsc{LeadingOnesTrailingZeros} (\lotz), and \textsc{OneJumpZeroJump} (\ojzj) problems for a general (even) number $m$ of objectives. We prove runtime guarantees showing that these algorithms compute the Pareto fronts of these problems in an expected time (number of function evaluations) that is linear in the size of the largest set of incomparable solutions (apart from small polynomial factors in the number $m$ of objectives and the bitstring length~$n$). For all benchmarks except \lotz and \ojzj with large jump size, this number coincides with or is close to the size of the Pareto front.
Since naturally the size of the Pareto front is a lower bound for these runtimes, these guarantees are tight apart from the small factors. In the bi-objective case ($m=2$) our results match the state-of-the-art bounds (apart from constant factors).

Together with the parallel and independent works on the \NSGAthree~\cite{OprisDNS24} and \SPEA~\cite{RenBLQ24}, these are the first that tight runtime guarantees for these MOEAs for general numbers of objectives, and they improve significantly over the previous results with their quadratic dependence on the Pareto front size. 

We prove a meta-theorem that gives upper bounds on the runtime of arbitrary iterative population-based MOEAs that (1)~fulfill the natural \emph{monotonicity property} that solutions only vanish from the population if they are dominated by another solution, and (2)~that have a positive probability to select any specific individual from the population and mutate it into any specific solution in Hamming distance~1.

We demonstrate that this meta-theorem can be applied to the (global) SEMO, \SMS, balanced \NSGAtwo, \NSGAthree, and \SPEA, giving novel bounds for all these algorithms. These bounds are mostly similar over the different algorithms, which extends this impression from the existing results for bi-objective optimization to many-objective optimization. However, our bounds also give some detailed information on how the parameters influence the runtime guarantee.
%We further believe that it can also be applied to variants of the \NSGAtwo that do not suffer from the problems detected in~\cite{ZhengD24many}, e.g., the \NSGAtwo with the tie-breaker proposed in~\cite{FortinP13}, and to the $(\mu+1)$ SIBEA~\cite{BrockhoffFN08}.
    
\section{Previous Work} 
\label{sec:prevWork}

The mathematical runtime analysis of randomized search heuristics is an active area of research for more than 30 years now, see~\cite{NeumannW10,AugerD11,Jansen13,ZhouYQ19,DoerrN20}. Since around 20 years, also the runtime of multi-objective evolutionary algorithms (MOEAs) has been analyzed with mathematical means. Starting with simple toy algorithms like the \emph{simple evolutionary multi-objective optimizer (SEMO)} \cite{LaumannsTZWD02} or the \emph{global SEMO (GSEMO)} \cite{Giel03}, the field has steadily progressed and is now able to also predict the runtimes of state-of-the-art algorithms such as the MOEA/D~\cite{LiZZZ16}, \NSGAtwo~\cite{ZhengD24}, \NSGAthree~\cite{WiethegerD23}, SMS-EMOA~\cite{BianZLQ23}, and \SPEA~\cite{RenBLQ24}.

Looking closer at the results obtained, we note that the vast majority of the runtime analyses of MOEAs consider only bi-objective problems. The sporadic results regarding more than two objectives appear less mature, and the runtime guarantees are far from the (mostly trivial) existing lower bounds. 

One natural reason for the additional difficulty of runtime analyses for many-objective problems, visible from comparing the proofs of results for two and for more objectives, is the richer structure of the Pareto front. In bi-objective problems, the Pareto front has a one-dimensional structure. Hence the typical runtime analysis first estimates the time to find some solution on the Pareto front and then regards how the MOEA progresses along the Pareto front in the only two directions available. For problems with more objectives, the Pareto front is higher-dimensional, and hence there are many search trajectories from the first solution on the Pareto front to a particular solution. 

The main runtime results for more than two objectives are the following. Already the journal version~\cite[Section~V]{LaumannsTZ04} of the first runtime analysis work on MOEAs~\cite{LaumannsTZWD02} contains two many-objective runtime results, namely proofs that the SEMO computes the Pareto front of \mcocz and \mlotz (which are the $m$-objective analogues of the classic \cocz and \lotz benchmarks) with problem size $n$ and a constant, even number $m\ge 4$ of objectives in an expected number of $O(n^{m+1})$ function evaluations. 
While the result for \mlotz naturally extends the $O(n^3)$ bound for the bi-objective \lotz problem, the same is not true for \mcocz, where the bi-objective runtime guarantee is $O(n^2 \log n)$. Considerably later, the bounds for \mcocz were slightly improved in~\cite{BianQT18ijcaigeneral}, namely to $O(n^m)$ for $m > 4$ and to $O(n^3\log n)$ for the special case $m=4$. As often in the runtime analysis of MOEAs, the complicated population dynamics prevented the proof of any interesting lower bounds, so only the trivial bound $\Omega(n^ {m/2} \Theta(m)^{-m/2})$, which is the size of the Pareto front for both problems, is known. 

Huang, Zhou, Luo and Lin \cite{HuangZLL21} analyzed how the MOEA/D~\cite{ZhangL07} optimizes the benchmarks \mcocz and \mlotz. As the MOEA/D decomposes the multi-objective problem into several single-objective subproblems and solves these in a co-evolutionary way, this framework is fundamentally different from the MOEAs regarded in this work, so we do not discuss these results in more detail.

In \cite{ZhengD24}, the expected runtime on the $\ojzj_k$ problem for arbitrary (even) numbers $m$ of objective was shown to be $O(n^k \frontOJZJ  \SOJZJ)$ for the global SEMO and $O(n^k \frontOJZJ \mu)$ for the \SMS \cite{BeumeNE07}, where $\frontOJZJ = (2\frac{n}{m} - 2k +3)^{m/2}$ and $\SOJZJ \le (2\frac{n}{m} +1)^{m/2}$ denote the size of the Pareto front and the size of the largest set of incomparable solutions for this problem, respectively, and $\mu \ge \SOJZJ$ denotes the size of the population of the \SMS.%
\footnote{We remark that \cite{ZhengD24} erroneously assume $\SOJZJ = \frontOJZJ$ and state the bounds as $O(n^k (\frontOJZJ)^2)$ and $O(n^k \frontOJZJ  \mu)$, respectively. They address this issue in their long version~\cite{ZhengD23smsarxiv}.}

Surprisingly, the \NSGAtwo~\cite{DebPAM02} has enormous difficulties with discrete many-objective problems. Zheng and Doerr \cite{ZhengD24many} showed that this prominent algorithm with any population size that is linear in the Pareto front size takes an at least exponential time (in expectation and with high probability) to optimize the \omm problem when the number of objectives is three or more. The proof of this result suggests that this is an intrinsic problem of the crowding distance, and that similar negative results hold for the other benchmark problems regarded in this work. The result of~\cite{ZhengD24many} was recently extended to the LOTZ benchmark~\cite{DoerrKK25}.

The \NSGAthree~\cite{DebJ14} copes better with more objectives. This was proven for the $3$-objective \omm problem with a runtime guarantee of $O(\mu n \log n)$ function evaluations (when the population size $\mu$ is at least the size $(\frac{n}{2}+1)^2$ of the Pareto front) \cite{WiethegerD23}.
In a very recent parallel and independent work, Opris et al.~\cite{OprisDNS24} proved that the monotonicity property of the \NSGAthree used in the analyses of the $3$-objective \omm problem holds for many benchmarks and numbers of objectives. They used this property to prove the runtime of the \NSGAthree (for a proper choice of reference points and population size $\mu$) to be $O(\mu n \log n)$ on \mOMM and \mcocz, and  $O(\mu n^2)$ on \mlotz, for arbitrary but constant and even $m$.

In another parallel and independent work, Ren et al.~\cite{RenBLQ24} proved a similar property for the \SPEA. Building on the proofs of Opris et al.~\cite{OprisDNS24}, they gave runtime bounds for arbitrary MOEAs that 
(1) have this monotonicity property, 
(2) for which there is $\mu\in\N$ such that the parent population is of size at most $\mu$ in each iteration,
(3) for which there is $c\in [\frac{1}{\mu}, O(1)]$ such that the offspring population is of size $c\mu$ in each iteration,
and, for the production of this offspring,
the MOEA uses (4) uniform selection and (5) bitwise mutation or one-bit mutation.%
\footnote{At the beginning of their Section~3, Ren et al.~\cite{RenBLQ24} state that $\mu$ is the parent population size, while in their theorems and proofs they use $\mu$ as the \emph{maximum} population size. We remark that their meta-theorem is applicable to MOEAs of changing population size with a fixed number of offspring (such as the (global) SEMO) only under the latter definition of $\mu$ as otherwise $c$ is not well defined.}
For such MOEAs, they proved that the expected number of fitness evaluations to find the full Pareto front is bounded from above by $O(\mu n \min\set{m \ln(n), n})$ for $\mOMM$, $O(\mu n^2)$ for \mlotz, and $O(\mu n^k \cdot \min\set{mn, 3^{m/2}})$ for \mJump, for arbitrary (even) number $m$ of objectives.
To apply this meta-theorem to the \SMS, \NSGAthree, and \SPEA, one has to take a population size of at least the maximum size $S$ of a set of incomparable solutions in the respective benchmark. For the GSEMO, one can use the theorem with $\mu=S$.%\merk{Considerably condensed, please check that I do not miss something important. S: checked, seems fine}
% They apply their meta-theorem to obtain bounds for the \SIBEA and \SMS (where $\mu \ge S$ is the population size and $c=\frac{1}{\mu}$), for the the (global) SEMO (where $\mu = S$ is an upper bound on the population size and $c=\frac{1}{\mu}$), the \NSGAthree (where $\mu \ge S$ is the population size and $c=1$), and the \SPEA (where $\mu \ge S$ is the parent population size and the offspring population size is $\lambda = c\mu$ for some constant~$c$), on the respective benchmark.%

\section{Our Contribution}\label{sec:ownResults}
We propose and prove a meta-theorem that gives runtime guarantees for a general and natural class of MOEAs on the most prominent benchmarks \mOMM, \mcocz, \mlotz, and \mJump. 
We apply this theorem to the (global) SEMO, \SMS, balanced \NSGAtwo, \NSGAthree, and the \SPEA. The resulting runtime guarantees, summarized in Table~\ref{tab:results}, significantly improve over the state-of-the art bounds.

% The bounds are the same for each algorithm except for a factor $Q$, which intuitively accounts for the population size and behavior and is defined as follows.
% For the (global) SEMO, let $Q\coloneqq S$, where $S$ is the of the largest set of incomparable solutions of the respective benchmark.
% For the \SMS and \NSGAthree, let $Q\coloneqq  \mu$, where $\mu \ge S$ is the population size, and for the \SPEA with offspring population size $\lambda$ let $Q \coloneqq  \max\set{\lambda, \mu}$.
\begin{table*}[]
\tiny
    \begin{center}
 \renewcommand{\arraystretch}{2}
        \begin{tabular*}{\textwidth}{l|c|l|c|c}
    \toprule
%        & \hphantom{m}Expected number of Evaluations\hphantom{m}
        & \makecell{asymptotic\\runtime guarantee} 
        & \makecell{GSEMO, SEMO} 
        & \makecell{\SMS, \\ \NSGAthree,\\ \SIBEA}
        &  \SPEA
        \\\midrule
        \mOMM
            & $(\frac{m}{\ln(n)}+\frac{m^2}{n}+1)n\ln(n) Q$ 
            & $Q = \SOMM = (\frac{2n}{m}+1)^{\frac{m}{2}}$
            & \multirow{4}{*}{$Q = \mu$} 
            & \multirow{3}{*}{$Q = \frac{\lambda}{n} + \mu$} 
            % & $Q = \mu$
            % & $Q = \frac{\lambda}{n} + \mu$ 
        \\\cline{1-3}
        % \\\hline
        \mcocz
            & $(\frac{m}{\ln(n)}+\frac{m^2}{n}+1)n\ln(n) Q$ 
            & $Q = \SCOCZ = (\frac{n}{m}+1)^{\frac{m}{2}}$
            & %$Q = \mu$
            & %$Q = \frac{\lambda}{n} + \mu$ 
        \\\cline{1-3}
        % \\\hline
        \mlotz 
            & $\max\set{\frac{n^2}{m}, mn\ln(\frac{n}{m}+1)} Q$
            & $Q = \SLOTZ = \Theta((\frac{2n}{m}+1)^{m-1})$
            & %$Q = \mu$
            & %$Q = \frac{\lambda}{n} + \mu$ 
        \\\cline{1-3}\cline{5-5}
        % \\\hline
        \mJump 
            & $m n^k Q$
            & $Q = \SOJZJ = O((\frac{2n}{m}+1)^{\frac{m}{2}})$
            & %$Q = \mu$
            & $Q = \frac{\lambda \ln(n)}{n^k} + \mu$ 
        \\\bottomrule
    \end{tabular*}
    \end{center}
    \caption{Overview on our results. In the first columns, we give the asymptotic orders of magnitude of our upper bounds on the expected number of fitness function evaluations several MOEAs take to compute the Pareto front on different benchmarks. These bounds are identical over the different algorithms apart from the algorithm-specific term~$Q$. The value $S_m^{B}$ (or $S_{m,k}^{B})$ refers to the size of a largest set of incomparable solutions of a benchmark $B$. 
    For the \SMS, \SIBEA, \NSGAthree, and \SPEA, we assume that the parent population size $\mu$ is at least $S_m^{B}$ (or $S_{m,k}^B$) for the respective benchmark $B$.
    The \NSGAthree is assumed to employ a suitable set of reference points. The \SPEA has the offspring population size $\lambda$ as additional parameter. For details, see Section~\ref{sec:otherAlgos}. For the SEMO, we have a stronger bound on the \mlotz benchmark, see Theorem~\ref{thm:results_semo}.      
    %\simon{include improved semo on \mlotz  }
    %Here, $S_\calB$ denotes the size of a largest set of pairwise incomparable solutions for a benchmark $\calB$.
    }
    \label{tab:results}
\end{table*}

We first compare our results against the state-of-the-art before ours and the parallel works~\cite{OprisDNS24,RenBLQ24}.
There, we give the first many-objective runtime guarantee for each of the analyzed algorithms on $\mOMM$. 
%\simon{wollen/können wir das so sagen? Ich hoffe ich übersehe da nichts, aber viele betrachten halt eher $\mcocz$ als $\mOMM$}. Ist schon OKee, wenn man nicht weiter darauf rum walzt.
For the \mcocz benchmark, the only previous results on algorithms studied in this work regard the SEMO algorithm. Here, our 
\[O((\tfrac nm + 1)^{m/2} (\tfrac{m}{\ln(n)}+\tfrac{m^2}{n}+1)n\ln(n))
\]
bound improves considerably over the classic $O(n^{m+1})$ runtime guarantee of~\cite{LaumannsTZ04} and the more recent $O(n^m)$ bound in~\cite{BianQT18ijcaigeneral}, e.g., by a factor 
%of  $O(\frac{n^{m/2-2}}{\ln(n)})$ for $m=\Theta(n)$ and
$\Theta(\frac{n^{m/2-1}}{\ln(n)})$ in the most relevant setting $m=\Theta(1)$.
% , as in this case $Q = \Theta(n^{m/2}\Theta(m)^{-m/2})$.

For the \mlotz benchmark,  the only existing runtime guarantee is $O(n^{m+1})$ for the SEMO and the case that $m$ is constant~\cite{LaumannsTZ04}. We extend this bound to the bound $O((\frac{2n}{m}+1)^{m-1} \max\set{\frac{n^2}{m}, mn\ln(\frac{n}{m}+1)})$ for general $m$, and prove comparable bounds for all other algorithms. For the SEMO, we improve the previous $O(n^{m+1})$ bound for constant $m$ to $\frac{n^2}{m} O(n \log n)^{m/2}$ for all $4 \le m = o(\sqrt{n / \log n}\,)$.
% The reason why here we do not obtain a considerable improvement is that on the \lotz benchmark, the SEMO, as visible in the proofs in~\cite{LaumannsTZ04}, drastically profits from its one-bit mutation operator. Thanks to it, it cannot create incomparable solutions that are not Pareto optimal. Consequently, the population at all times either consists of a single individual not on the Pareto front, or of a subset of the Pareto front. For this reason, the SEMO (as only of the algorithms regarded here) has the particular advantage that its population size is bounded by the size $(\frac {2n}m +1)^{m/2}$ of the Pareto front, whereas, e.g., the GSEMO can attain populations that witness the largest incomparable set, having size $(\frac {2n}{m} + 1)^{m-1}$~\cite{OprisDNS24}.

%In the case of the SEMO optimizing \mlotz with constant $m$, we have $Q = \Theta(n^{m-1})$ and thus match the existing bound of $O(n^{m+1})$~\cite{LaumannsTZ04}. 
For the \mJump benchmark, a runtime guarantee of $O((\frac{2n}{m}-2k+3)^m n^k)$ was proven for the global SEMO and the \SMS (with optimal parameters), see~\cite{ZhengD24}. We improve this to $O((\frac{2n}{m}+1)^{m/2} m n^k)$, valid for all algorithms (with suitable parameters) except the SEMO (which cannot optimize this benchmark).
%by a factor of $\frac{\frontOJZJ}{m}$, where $\frontOJZJ = (2\frac{n}{m}-2k+3)^{m/2}$ is the size of the Pareto front.

We note that all previous bounds for $m>4$ objectives are quadratic in the Pareto front size (or even quadratic in the size of the largest incomparable set), times some small polynomial in $m$ and $n$. Our results show that this quadratic dependence is not necessary and merely stems from the difficulty to analyze many-objective MOEAs: For each considered benchmark except \mlotz (and \mJump with particularly large values of $k$) we give a runtime guarantee that is linear in the Pareto front size, times some small polynomial in $m$ and $n$.

We now discuss the two works~\cite{OprisDNS24} and~\cite{RenBLQ24} that were conducted in parallel to ours. Opris et al.~\cite{OprisDNS24} only analyze the \NSGAthree and the benchmarks \mOMM, \mcocz, and \mlotz for constant values of $m$; here their bounds agree with ours. It is possible to extend their proofs to non-constant values of $m$, this would lead to bounds larger than ours by roughly a factor of $m$.
%, and a factor of $\Omega(1+\frac{m\ln(\frac{n}{m}+1)}{n})$ for \mlotz. 
%\simon{z.B. für \mOMM sollte es sowas wie $\Omega(m (\frac{n+1}{n})^m)$ sein, aber ich denke der Satz oben reicht?}
% For \mOMM, this factor is $\Omega(m (\frac{n+1}{n})^m)$. Similarly, their bound for \mcocz would increase by a factor in $\omega(m)$.
% For \mlotz, the factor is $\Omega(1+\frac{m\ln(\frac{n}{m}+1)}{n})$. 
%We note that our guarantees are faster by a factor at least $m$ in all but the extreme cases on \mlotz with very large $m$, where we match their guarantee, asymptotically. \merk{verstehe ich nicht. Wie kann das sein wenn der Faktor fast nichts ist? S: ich auch gerade nicht. Muss ich nochmal nachrechnen, wie der Faktor genau aussieht}

The meta-theorem provided by Ren et al.~\cite{RenBLQ24} holds for a large class of MOEAs similar to ours.
Using independently developed proofs that employ similar but more careful techniques, our meta-theorem captures a more general class of MOEAs by eliminating the dependency on specific selection and mutation operators and provides tighter guarantees in most settings:
For $\mOMM$, the bound by \cite{RenBLQ24} is tighter than ours only in the exotic case of $m= \Omega(\frac{n}{\ln(n)})$ (by a factor of $\ln(n)$). Otherwise, our bound is tighter by a factor of $\frac{m}{1+m/\ln(n)}\approx \ln(n)$. 
For \mlotz, our guarantee improves over their bound by a factor of $m$ if $m^2 \le \frac{n}{\ln(n/m)}$ and at least match their bound in the exotic remaining cases.
For \mJump, we improve over their bound by a factor of $\min\set{n, \frac{3^{m/2}}{m}}$.

We acknowledge that when applying our meta-theorem to the \SPEA and \NSGAthree, we heavily rely on the monotonicity properties proven in these two works \cite{RenBLQ24} and \cite{OprisDNS24}. In contrast, the general approach of dealing with many objectives is novel and, as visible from the stronger results for larger numbers $m$ of objectives, different.

In the last technical section, we demonstrate the flexibility of our meta-theorem by extending our results to the heavy-tailed mutation operator proposed in~\cite{DoerrLMN17}. As in many previous works, this mutation operator does not affect the asymptotic runtime guarantees for unimodal functions (or multi-objective problems with unimodal objectives), but gives a speed-up  by a factor of $k^{\Omega(k)}$ on the \mJump benchmark. 

We note without explicit proof (but this is easily visible from our proofs) that all our results also hold for variants of the algorithms that use crossover with constant probability less than one. 

This work significantly extends the conference version~\cite{WiethegerD24}, among others, by explicitly formulating the meta-theorem, by analyzing the balanced \NSGAtwo and the \SPEA, by including heavy-tailed mutation, and by proving a strongly improved asymptotic runtime guarantee for the SEMO on the \mlotz benchmark when $m$ is not too large.

\section{Preliminaries}
\label{sec:prelims}
Let $\N = \set{1,2,\ldots}$ and for $n\in \N$, let $[n] = \set{1,\ldots, n}$.
Whenever we are speaking of how close one bitstring is to another, we are referring to their Hamming distance.
Many of our proofs rely on two well-known bounds in probability theory, which we give here for completeness and reference in our proofs by name.
First, a \emph{union bound} states that for some finite set of events $\set{E_1,\ldots, E_\ell}$ the probability that at least one of the events happens is at most the sum over the individual probabilities, that is, $\Pr[\bigcup_{i\in[\ell]}E_i] \le \sum_{i\in[\ell]} \Pr[E_i]$.

Second, we employ variants of the \emph{multiplicative Chernoff bound}.
Let $X_1, \ldots, X_\ell$ be independent random variables with values in $\set{0,1}$ and let $X = \sum_{i\in[\ell]} X_i$.
Then for any $0<\delta <1$ we have 
\begin{equation*}
    \Pr\left[X \le (1-\delta)E[X]\right] \le \exp\left(-\tfrac{1}{2} \delta^2 E[X]\right).
\end{equation*}
In particular, we have $\Pr[X \le \frac{1}{2}E[X]] \le \exp(-\frac{1}{8} E[X])$.
Further, for any $0 \le \delta \le 1$ we have 
\begin{equation*}
    \Pr\left[X \le (1+\delta)E[X]\right] \le \exp\left(-\tfrac{1}{3} \delta^2 E[X]\right).
\end{equation*}
The above bounds can be found in, e.g., \cite[Lemma 1.5.1, Theorems~1.10.1 and 1.10.12]{Doerr20bookchapter}.

\subsection{Multi-Objective Optimization}
\label{sec:moo}
Let $m\in \N$. 
An $m$-objective function $f$ is a tuple $(f_1, \ldots, f_m)$ such that $f_i\colon \Omega \rightarrow \R$ for some search space $\Omega$, for all $i\in [m]$.
We define the objective value of $x\in\Omega$ to be $f(x) = (f_1(x),\ldots,f_m(x))$.
There is usually no solution that maximizes all $m$ objective functions at the same time. 
Therefore we resort to the solution concepts of domination and Pareto optimality.
For $x,y\in \Omega$ we write $x \succeq y$ if and only if $f_i(x) \ge f_i(y)$ for all $i\in [m]$ and say that $x$ \emph{dominates} $y$.
If additionally $f_j(x) > f_j(y)$ for some $j\in[m]$, we say that $x$ \emph{strictly dominates} $y$ and write $x \succ y$. 
A solution $x\in \Omega$ is Pareto-optimal if it is not strictly dominated by any other solution. The \emph{Pareto front} is the set of objective values of Pareto-optimal solutions.

Given an algorithm and an objective function, we are interested in the number of function evaluations until the population \emph{covers} the Pareto front, that is, until for all values $p$ on the Pareto front there is a solution $x$ in the population such that $f(x)=p$.
We note that for each MOEA $\calA$ analyzed in this work there is a $\lambda\in\N$ such that $\calA$ creates precisely $\lambda$ new individuals per iteration and thus performs $\lambda$ function evaluations per iteration.
Thus we usually analyze the number of iterations until the Pareto front is covered. To obtain the runtime in terms of fitness evaluations, we multiply this value by $\lambda$, and add the cost for evaluating the initial population (1 evaluation for the SEMO and GSEMO, $\mu$ evaluations for the \NSGAtwo and \SMS with population size $\mu$, and $\max\set{\lambda, \mu}$ evaluations for the \SPEA).

\subsection{Benchmarks}
\label{sec:benchmarks}
The four established multi-objective benchmarks that we consider are defined on the search space of bitstrings of length $n$ for some $n\in\N$; they are defined for even numbers $m$ of objectives. Let $m' = m/2$ and assume that $m'$ divides $n$. Then our benchmarks are obtained by partitioning the individuals (which are bitstrings of length~$n$) into $m'$ blocks of size $n' = \frac{n}{m'}$ and applying a bi-objective function to each block.
For an individual $x\in \set{0,1}^n$ and $i\in [m']$, we define $x^i$ to be the $i$th block of $x$, that is the substring from $x_{n'(i-1)+1}$ to $x_{n'i}$. 
We define $|x^i|_1 = \sum_{j=(i-1)n'+1}^{in'} x_j$ (and $|x^i|_0 = n'-|x^i|_1$) to denote the number of 1-bits (and 0-bits) in the $i$th block.

\paragraph*{$m$OneMinMax (\mOMM)}
The objective function \mOMM translates the well-established \onemax benchmark in a setting with $m=2m'$ objectives for some $m'\in \N$.
Intuitively, the bitstring is divided into $m'$ equally sized blocks that each contribute two objectives, the number of 1-bits and the number of 0-bits in that block.
The bi-objective case of $m'=1$ was proposed  by \cite{GielL10} and later generalized to arbitrary $m'$ \cite{ZhengD24many}.
Let $m',n'\in \N$ and $n = m'n'$.
For all $x\in \set{0,1}^n$, define $\mOMM(x) = (f_1(x),\ldots,f_{m}(x))$, where for all $i\in[m']$ we have
\begin{align*}
    f_{2i}(x) =  |x^i|_1 \text{\quad and \quad}
  f_{2i-1}(x) =  n'-|x^i|_1.
\end{align*}

The benchmark \mOMM is special in the sense that each possible objective value lies on the Pareto front, {giving $\SOMM = \frontOMM = \left(n'+1\right)^{m'}$ for the size $\SOMM$ of a largest set of incomparable solutions and the size $\frontOMM$ of the Pareto front.}

\paragraph*{$m$CountingOnesCountingZeros (\mcocz)}
The \cocz benchmark and its multi-objective variant \mcocz \cite{LaumannsTZ04} are closely related to $(m)$\omm. 
However, the objectives cooperate on the first half of the bitstring and only the second half is evaluated just like for \mOMM.
Formally, let $m',n'~\in~\N$, $n = 2bm'$, and $m=2m'$.
Then for all $x\in \set{0,1}^n$, define $\mcocz(x) = (f_1(x),\ldots,f_{m}(x))$, where for all $i\in[m']$ we have
\begin{align*}
f_{2i}(x) &=  \sum_{j=1}^{bm'} x_j + \sum_{j=(i-1)b+1}^{ib} x_{bm'+j} \text{\quad and}\\ 
f_{2i-1}(x) &=  \sum_{j=1}^{bm'} x_j + \sum_{j=(i-1)b+1}^{ib} (1 - x_{bm'+j}). 
\end{align*}
We immediately see that exactly the solutions with the first half of the bits equal to one are on the Pareto front. Consequently, for $n'=\frac{n}{m'}=2b'$, $\frontCOCZ=(\frac{n'}{2}+1)^{m'}$ is the size of Pareto front for the \mcocz problem. Since any two solutions that agree in the second half of the bit string are comparable (at least one dominates the other), this number is also the maximum size {$\SCOCZ = \frontCOCZ$} of any set of pairwise non-dominating individuals. See~\cite{OprisDNS24} for a formal proof of these facts. 

\paragraph*{$m$\textsc{LeadingOnesTrailingZeros} (\mlotz)}
The \mLOTZ objective function is the many-objective variant of the bi-objective \lotz benchmark \cite{LaumannsTZ04}.
Intuitively, it has two objectives per block, one being the number 1-bits up to the first 0-bit and the other being the number of 0-bits behind the last 1-bit.
Formally, let $m',n'\in \N$, $n = m'n'$, and $m=2m'$.
Then for all $x\in \set{0,1}^n$, define $\mLOTZ(x) = (f_1(x),\ldots,f_{m}(x))$, where for all $i\in[m']$ we have
\begin{align*}
f_{2i}(x) &=  \sum_{j=(i-1)n'+1}^{in'}\prod_{j'=1}^j x_{j'}  \text{\quad and}\\ 
f_{2i-1}(x) &=  \sum_{j=(i-1)n'+1}^{in'}\prod_{j'=j}^{in'} (1- x_{j'}).
\end{align*}
% % above for short version

Observe that $\frontLOTZ=\frontOMM=(n'+1)^{m'}$ is the size of the Pareto front for the \mLOTZ problem.
However, the size $\SLOTZ$ of a largest set of pairwise non-dominating individuals is asymptotically tightly bounded by $\SLOTZ \le (n'+1)^{2m'-1}$ \cite{OprisDNS24} and thus almost quadratic in $\frontLOTZ$.

\paragraph*{$m\textsc{OneJumpZeroJump}_k$ (\mJump)}
The objective function $\mJump$ is the recently introduced many-objective variant \cite{ZhengD24} of the bi-objective \ojzj benchmark \cite{DoerrZ21aaai}.
It has two objectives per block, one for the number of 1-bits and one for the number of 0-bits. However, it has a fitness valley with decreasing objective value if the number of 0-bits or 1-bits is in $[k]$, where the \emph{jump size} $k$ is a parameter of the benchmark.
Formally, let $m',n'\in \N$, $n = m'n'$ and $m=2m'$.
Then for all $x\in \set{0,1}^n$, define $\mJump(x) = (f_1(x),\ldots,f_{m}(x))$, where for all $i\in[m']$ we have
    $f_{2i}(x) = \jump_k(x^i)$ and $f_{2i-1}(x) = \zerojump_k(x^i)$ with 
\begin{align*}
    \jump_k(x) = \begin{cases} 
                        |x|_1 +k, \text{\quad if } |x|_1 \le n'-k \text{ or } |x|_1 = n';\\  
                        n'-|x|_1, \text{\quad else;}  
                  \end{cases} \\
    \zerojump_k(x) =  \begin{cases} 
                         |x|_0 +k, \text{\quad if } |x|_0 \le n'-k \text{ or } |x|_0 = n';\\  
                        n'-|x|_0, \text{\quad else.}  
                      \end{cases} 
\end{align*}

As in \cite{DoerrZ21aaai}, we assume that $2 \le k \le \frac{n'}{2}$.
The 2-objective \ojzj benchmark with jump size $k$ has a Pareto front of size $n-2k+3$ \cite{DoerrZ21aaai}. 
Thus, the Pareto front of \mJump, which corresponds to \ojzj in $m'$ individual blocks of size $n'$, is of size $\frontOJZJ = \left(n'-2k+3\right)^{m'}$.
As the objective value is characterized by the number of 1-bits in each block, the size $\SOJZJ$ of a largest set of pairwise non-dominating individuals is upper bounded by $\SOJZJ \le \left(n'+1\right)^{m'} = \SOMM$.
Note that $\SOJZJ \approx \frontOJZJ$ for small values of $k$, which is the most interesting case as otherwise usually very large runtimes are observed.

\section{Mathematical Analyses of GSEMO}
\label{sec:analysis_gsemo}

We start our mathematical runtime analyses by first only regarding the global SEMO (GSEMO) algorithm. This algorithm has been intensively studied for more than twenty years~\cite{Giel03}, so we expect more readers to be familiar with this algorithm than with the other algorithms studied in this work, for which the first runtime analyses were conducted only in the very recent past~\cite{ZhengLD22,BianZLQ23,WiethegerD23,RenBLQ24}.

The global SEMO starts the first generation with a single, random individual in the population.
In each iteration, it uniformly at random chooses an individual from the current population and mutates it to a new solution $x'$ by independently flipping each bit of the parent with some probability~$p$ (bitwise mutation). 
Here we assume the conventional mutation rate $p = \frac{1}{n}$.
The GSEMO then removes all solutions from the population that are dominated by $x'$ and adds $x'$ to the population if and only if it is not strictly dominated by a solution in the population, see Algorithm~\ref{alg:gsemo}.
This way, the population stores exactly one individual for each encountered objective value that is not strictly dominated by any other encountered objective value.
\begin{algorithm2e}[ht]%
\Input{%
objective function $f=(f_1,\ldots, f_m)$,\\ length of bitstrings $n\in\N$
\\
}
    Generate $x_0\in \set{0,1}^n$ uniformly at random and let $P_0 \coloneqq \set{x_0}$\\
    \For{$t = 1, 2, \ldots$}{
        Select $x$ from $P_{t-1}$ uniformly at random and let $x'\coloneqq \mutate(x)$\\
        $P_t \coloneqq \set{p \in P_{t-1} \mid x' \not\succeq p}$\\
        \If{there is no $p \in P_t$ such that $p \succ x'$}{
                $P_t \coloneqq P_t \cup \set{x'}$
            }
        }
\caption{(Global) SEMO}
\label{alg:gsemo}
\end{algorithm2e}%
Observe that by evaluating the objective function only once after a new individual is created and storing the value for future comparisons, the number of evaluations is exactly the number of generations (plus $1$ for the initial individual). 

We start our contribution by giving upper bounds on the optimization time of the GSEMO on the four benchmarks. 
Afterward, in Section~\ref{sec:otherAlgos}, these results are transferred to other MOEAs as well.

Recall that for the benchmarks considered in this work the size of the largest set of incomparable solutions is $\SOMM,\SCOCZ,\SLOTZ,$ or $\SOJZJ$, respectively.
Consequently, these are upper bounds on the population size in any iteration of the GSEMO.
Our analyses heavily build on progress made by flipping single bits each iteration.
We thus note that in each generation the probability of the GSEMO creating any specific solution $y$ that has a Hamming distance of one to at least one solution $x$ in the current population is at least 
\[\frac{1}{S} \cdot \frac{1}{n} \cdot \nootherflip \ge \frac{1}{en S } \eqqcolon q,\] 
namely by choosing $x$ as parent from the population of size at most $S$ (the size of a largest set of incomparable solutions), flipping the correct bit, and not flipping any other bit.

\subsection{$m$\textsc{OneMinMax}} 
\label{sec:mOMM}

Our main result in this subsection is the following runtime guarantee for the GSEMO on the many-objective OneMinMax problem. 
%\simon{check each theorem full setting $m, m', n$}

\begin{theorem}\label{thm:upperBoundOMM}
Let $m',n'\in \N$, $n=m'n'$, and  $m= 2m'$. Consider the GSEMO optimizing \mOMM.
Let $F$ denote the number of fitness evaluations until the population covers the complete Pareto front and let
\begin{align*}
    t \coloneqq \left(\frac{\ln(2)m'+2}{\ln(n)} + 16 \frac{{m'}^2+2m'}{n} + 2\right) e n \ln(n) \SOMM.
\end{align*}
Then $F \le t+1$ with probability at least $(1-\frac{1}{n})^2 = 1 - \Theta(\frac 1n)$ and $E[F] \le (1-\frac{1}{n})^{-2}t+1$.
% Let $T$ denote the number of iterations until the population covers the complete Pareto front and let $t$ be
% \begin{align*}
%      \left(\frac{\ln(2)m'+2}{\ln(n)} + 16 \frac{{m'}^2+2m'}{n} + 2\right) e \SOMM n \ln(n) +1.
% \end{align*}
% Then $T \le t$ with high probability and $E[T] \le (1-\frac{1}{n})^{-2}t$.
\end{theorem}

Since, different from~\cite{OprisDNS24} we aim at results also for super-constant $m$, and different from \cite{RenBLQ24}, we aim a results with an as good as possible dependence on $m$, we have to resort to a more complex proof strategy. In particular, to avoid an often unnecessary loss by a factor of $m$, we shall find arguments that distill a parallel progress in the $m'$ blocks of the benchmarks. 
%\simon{$m=\Omega(n/\log n)?$ Streng genommen nehmen \cite{RenBLQ24} nicht $m$ als konstant an, für große $m$ ist nur ihre Schranke schlechter }  \merk{OK, hatte nicht gesehen, dass \cite{RenBLQ24} superkonstantes $m$ zulassen. Das müssen wir dann anpassen. Für \cite{OprisDNS24} hatte ich (grob) gecheck dass deren Beweis auch für gro\ss e $m$ geht, aber dass sie halt einen Faktor von $m$ dazukriegen. Das ist dann einen Faktor von $m$ schlechter als wir, ausser für grosse $m$, wo wir auch noch etwas Dreck dazukriegen. Wenn du das geeignet anpassen könntest, wär das cool. Ich hab grad 4 Kinder um mich rum, da geht grad nicht viel...}
%\simon{Ein bisschen was dazu hab ich in der Our Contribution Section geschrieben. Wie viel wollen wir davon hier nochmal wiederholen?}
% B: Ich hab hier jetzt ganz knapp was geschrieben.

Interestingly, our new proof strategy also suggests what is the real difficulty of solving \mOMM, namely finding the (relatively) few corner individuals (individuals which are all-ones or all-zeros in each block). We analyze this time in Lemma~\ref{lem:walkToCorners}. The bound we prove, which we suspect to be tight, is the asymptotically dominant time in our final runtime guarantee except for the quite exotic situation that $m = \Omega(n / \log n)$. In contrast, once all $2^{m'}$ corners are found, finding the other $\Omega((n')^{m'})$ individuals is relatively fast (see Lemma~\ref{lem:coverInwards}), taking time $O(n \SOMM)$  when $m = O(\sqrt{n / \log n})$ -- for comparison, even for finding a single corner we do not have a better guarantee than $O(n \log(n) \SOMM)$.  

Formally, we define the set \corners of ``corners'' to contain all bitstrings for which each of the $m'$ blocks consists either only of 1-bits or only of 0-bits.
Observe that $\abs{\corners} = 2^{m'}$ and that each bitstring in \corners is the unique individual of the respective objective value.
% To prove our runtime guarantee, we separately bound the time until all individuals in \corners are in the population, see Lemma~\ref{lem:walkToCorners}, and then analyze how from this point on all other individuals are generated, see Lemma~\ref{lem:coverInwards}.
For both Lemma~\ref{lem:walkToCorners} and \ref{lem:coverInwards}, recall that $q = \frac{1}{en\SOMM}$ is a lower bound on the probability to sample any specific bitstring that has Hamming distance oneto some bitstring in the population in any given iteration.
% The bounds of the two lemmas are then combined into our final runtime guarantee by Theorem~\ref{thm:upperBoundOMM}.

\begin{lemma}\label{lem:walkToCorners}
Let $m',n'\in \N$, $n=m'n'$, and  $m= 2m'$. Consider the GSEMO optimizing \mOMM and let $T$ denote the number of iterations until the population contains \corners. 
Let 
  \[t_1 \coloneqq \left(\ln(2) \frac{m'}{\ln(n)}+2\right) \frac{1}{q} \ln (n).\] 
  Then $T \le t_1$ with probability at least $1-\frac{1}{n}.$
\end{lemma}
\begin{proof}
We first prove a tail bound for the time $T_C$ until the population contains the corner bitstring $1^n$.
By the symmetry of \mOMM, this bound applies to all other elements in \corners as well.
Hence, applying a union bound over the tail bounds for the individual elements in \corners yields a bound on the time until all elements are covered.

Let $x$ be a member of the population with maximum number of 1-bits.
Let $i = n- |x|_1$ be the number of 0-bits of $x$.
Then the probability of sampling an individual that has $i-1$ many 0-bits in the next iteration is at least $iq$, by creating one of the $i$ bitstrings obtained by flipping a 0-bit in $x$.
After at most $n$ such iterations the corner $1^n$ is sampled.
% Hence,    $E[T_C] \le \sum_{i=1}^n \frac{1}{iq}.$
    
For $i \in [n]$, let $X_i$ be independent geometric random variables, each with success probability $iq$, and let $X=\sum_{i=1}^n X_i$.
Then $X$ stochastically dominates $T_C$, and thus a tail bound for $X$ also applies to $T_C$.
By Lemma~4 of~\cite{DoerrD18} (also found as Theorem~1.10.35 in \cite{Doerr20bookchapter}), a tail bound for sums of geometric random variables with harmonic success probabilities, for all $\delta \ge 0$ we have 
    \begin{align*}
        \Pr[T_C\ge (1+\delta)q^{-1} \ln (n)]
        &\le \Pr[X\ge (1+\delta)q^{-1} \ln (n)] \le n^{-\delta}.
    \end{align*}
 
 By the symmetry of the problem and operators, the bound holds for all elements in \corners.
Letting $\delta = m'\log_n(2)+1$ and by applying a union bound we obtain
\begin{align*}
\Pr[T\le t_1] &\ge 1-\abs{\corners}\cdot \Pr[T_C
  \ge (\delta+1)q^{-1} \ln (n)]\\
 &\ge 1-2^{m'} n^{-\delta}
  = 1-\tfrac{1}{n}.
\end{align*}
We note that this applies for arbitrary starting configurations, as all that we assumed about the initial population was that it is non-empty.
\end{proof}

For Lemma~\ref{lem:coverInwards}, we introduce two kinds of notation:
If a bitstring has $a_i$ bits of value~1 in the $i$th block for all $i\in [m']$, we call it an \gabstring. 
Further, for this benchmark only, we abbreviate the notation of vectors of the Pareto front from $(n'-a_1, a_1, n'-a_2, a_2, \ldots, n'-a_{m'}, a_{m'})$ to $(a_1, a_2, \ldots,a_{m'})$.

\begin{lemma}\label{lem:coverInwards}
Let $m',n'\in \N$, $n=m'n'$, and  $m= 2m'$. Consider the GSEMO optimizing \mOMM starting with a population that contains at least all individuals in \corners. 
Let $T$ denote the number of iterations until the population covers the complete Pareto front and let 
\begin{align*}
    t_2 \coloneqq \max\left\{1, 8 \frac{m'(m'\ln\left(n'+1\right)+\ln(m')+\ln(n))}{n}\right\}
    \frac{2}{q}.
\end{align*}
Then $T \le t_2$ with probability at least $1-\frac{1}{n}$.
\end{lemma}
\begin{proof}
Consider any objective value $v=(a_1,a_2,\ldots,a_{m'})$ on the Pareto front. 
Let $c_0\in \corners$ be any closest corner to an \gabstring.
We bound the time until an \gabstring is generated by bounding the time until a marked individual $c$ becomes an \gabstring.
Let initially $c = c_0$.
Whenever a Hamming neighbor $c'$ of $c$ is sampled that is closer to any \gabstring than $c$, we update $c$ to be $c'$.
Also, when $c$ is removed from the population, we replace it by any $c' \succeq c$ in the population, which for this benchmark implies $\mOMM(c)=\mOMM(c')$.
The time until the population contains an \gabstring is at most the time until $c$ is an \gabstring.

We first bound the probability that after $t_2$ iterations there are exactly $a_i$ bits of value~1 in the $i$th block of $c$, for any fixed $1\le i \le m'$. 
By symmetry, suppose without loss of generality that $a_i \le \frac{n'}{2}$ and that the $i$th block of $c_0$ is $0^{n'}$.
All $c$ we will encounter have between $0$ and $a_i$ bits with value~1.
The probability of sampling a Hamming neighbor of $c'$ that has flipped one of the at least $\frac{n'}{2}$ bits of value~0 in the $i$th block of $c$ and no other bit is at least 
\[p \coloneqq q\cdot \frac{n'}{2}.\]
After $a_i \le \frac{n'}{2}$ such iterations, $c$ contains the correct number of 1-bits in the $i$th block. 
Thus, for the time $T_i$ until the $i$th block of $c$ is correct we have
    $E[T_i] \le \frac{a_i}{p}.$  %$ \le \frac{n'}{2} \cdot 2em'\SOMM = en\SOMM.\]
For $j \in [\floor{t_2}]$, let $X_j$ be independent random variables, each with a Bernoulli distribution with success probability $p$.
Let $X=\sum_{j=1}^{\floor{t_2}} X_j$.
Then, due to stochastic domination, $\Pr[T_i > t_2] \le \Pr[X \le a_i -1]$.
By observing $E[X]= p\floor{t_2} \ge n'-1$ we have
\begin{align*}
    \Pr[X \le a_i -1] 
    \le \Pr\left[X \le \tfrac{n'}{2}-1\right] 
    \le  \Pr\left[X \le \tfrac{1}{2} E[X]\right].
\end{align*}
Applying a multiplicative Chernoff bound from Section~\ref{sec:prelims} yields
\begin{align*}
    \Pr[T_i > t_2] 
    &\le \exp\left(-\tfrac{1}{8}  E[X]\right)\\
    &\le \exp\left(-\ln(m')-m'\ln\left(n'+1\right)-\ln(n)\right).
\end{align*}
Using a union bound over all $m'$ blocks gives that any fixed objective value on the Pareto front is not sampled in $\floor{t_2}$ iterations with probability at most 
\begin{align*}
     \exp\left(-\ln(m')-m'\ln\left(n'+1\right)-\ln(n)\right)m'
    = \exp\left(-m'\ln\left(n'+1\right)-\ln(n)\right).
\end{align*}

Let $E$ denote the event that after $\floor{t_2}$ iterations there is still an objective value $(a_1,a_2,\ldots,a_{m'})$ on the Pareto front such that the respective individual $c$ does not contain the correct number of 1-bits in any block.
By applying a union bound we have
\begin{align*}
  \Pr[E] &\le \frontOMM \cdot \exp\left(-m'\ln\left(n'+1\right)-\ln(n)\right)
    =  \exp(-\ln(n)) = \tfrac{1}{n}
\end{align*}
by observing $\frontOMM = \exp(m'\ln(n'+1))$. 
\end{proof}

From the two partial results just proven, we can now derive our main result for the many-objective \omm problem.

\begin{proof}[Proof of Theorem~\ref{thm:upperBoundOMM}]
Let $t_1$ and $t_2$ be as in Lemmas~\ref{lem:walkToCorners} and \ref{lem:coverInwards}, respectively, and let $T$ denote the number of iterations until the population covers the complete Pareto front.
The lemmas yield that $T \le t_1 + t_2$ with probability at least $(1-\frac{1}{n})^2$.

We have that $t \ge t_1 + t_2$ 
by observing that
$t_2 \le (\frac{2}{\ln(n)} + 16\frac{{m'}^2+2m'}{n}) \frac{\ln(n)}{q}$, where we use that $\max\set{\ln(n'+1,\ln(m')} \le \ln(n)$ if $m' \ge 2$. 
If $m' = 1$, we have $t_2 = \frac{2}{q}$ and the above inequality immediately holds.
As each iteration as well as the initialization takes one fitness evaluation, we have $F \le t+1$ with probability at least $(1-\frac{1}{n})^2$.

We employ a simple restart argument to obtain an upper bound on the expected value of $T$, analyzing the success of each sequence of $t$ iterations separately.
Each sequence of $t$ iterations fails to cover the Pareto front with probability at most $1-(1-\frac{1}{n})^2$.
Due to the convergence of the geometric series we have 
    \[E[T] 
    \le \sum_{i=0}^{\infty}  \left(1-\left(1-\tfrac{1}{n}\right)^2\right)^i t
    = \left(1-\tfrac{1}{n}\right)^{-2} t.
    \] 
The statement follows as each iteration takes one fitness evaluation (of the new solution) and an additional evaluation is required for the initial solution in $P_0$, so $F\le T+1$. 
% Let $t_1$ and $t_2$ be the optimization times (in iterations) in Lemmas~\ref{lem:walkToCorners} and \ref{lem:coverInwards}.
% We have 
% $t \ge t_1 + \ceil{t_2}$ by observing $\ln(n) \ge \ln(m')$ and $\ln(n)+1 \ge \ln(n+1) \ge \ln\left(n'+1\right)$.
% Thus, $T \le t$ with a high probability of at least $(1-\frac{1}{n})^2$.
% We employ a simple restart argument to obtain an upper bound on the expected value of $T$, analyzing the success of each sequence of $t$ iterations separately.
% Each sequence of $t$ iterations fails to cover the Pareto front with probability at most $1-(1-\frac{1}{n})^2$.
% Due to the convergence of the geometric series we have 
%     \[E[T] 
%     \le \sum_{i=0}^{\infty}  \left(1-\left(1-\frac{1}{n}\right)^2\right)^i t
%     = \left(1-\frac{1}{n}\right)^{-2} t.
% \qedhere
%     \] 
\end{proof}

\subsection{$m$\textsc{CountingOnesCountingZeros}}
\label{sec:mCOCZ}
Due to the similarity between \mOMM and \mcocz, our previous proofs can be adapted to also work for \mcocz.
We first show that with probability at least $1-\frac{1}{n}$ the population after $2 q \ln(n)$ iterations contains an individual such that the cooperative, first half is maximized. From this point on, we employ the same ideas as for Theorem~\ref{thm:upperBoundOMM} by only considering individuals with maximum cooperative part for the progress. A possible slowdown from individuals not on the Pareto front is accounted for via the probability $q$ to generate a particular Hamming neighbor of an existing solution.
\begin{theorem}\label{thm:upperBoundCOCZ}
% Let $m'\in \N$ and  $m= 2m'$. Consider the GSEMO optimizing \mcocz.
% Let $T$ denote the number of iterations until the population covers the complete Pareto front and let $t$ be
%  \begin{align*}
%     \left(\frac{\ln(2)m'+2}{\ln(n)} + 16 \frac{{m'}^2+2m'}{n} + 4\right) e \SCOCZ n \ln(n) +1.
% \end{align*}
% Then $T\le t$ with high probability and $E[T] \le (1-\frac{1}{n})^{-3} t$.
Let $m',n'\in \N$, $n=m'n'$, and  $m= 2m'$. Consider the GSEMO optimizing \mcocz.
Let $F$ denote the number of fitness evaluations until the population covers the complete Pareto front and let 
 \begin{align*}
   t \coloneqq \left(\frac{\ln(2)m'+2}{\ln(n)} + 16 \frac{{m'}^2+2m'}{n} + 4\right) e n \ln(n)  \SCOCZ.
\end{align*}
Then $F\le t+1$ with probability at least $(1-\frac 1n)^3 = 1-\Theta(\frac 1n)$ and $E[F] \le (1-\frac{1}{n})^{-3} t+1$.
\end{theorem}
\begin{proof}
    We first consider the number $T_0$ of iterations until an individual is generated that optimizes the cooperative part.
    Recall that $b=n'/2$ is the number of bits in each of the $m$ blocks.
    For an individual $x$, we define $\half(x) = \sum_{i=1}^{bm'} (1-x_i)$ to be the number of zeros in the cooperative part.
    For a population~$P$, let $\half(P) = \min_{x\in P} \half(x)$.
    Then $T_0$ is the time until $\half(P) = 0$.
    Observe that $x \succeq y$ implies $\half(x) \le \half(y)$.
    Thus, $\half(P)$ is non-increasing over the generations. 
    Fix any generation with population $P$ and let $\half(P) = i$, which is witnessed by some $x\in P$ with $\half(x) = i$.
    Then with probability at least $iq$ 
    this iteration produces a (non-dominated) offspring $x'$ with $\half(x') = i-1$, namely by creating an offspring that differs from $x$ precisely in any one of the $i$ bits of value 0 in the cooperative parts.
    As $\half(P)$ is non-increasing, after at most $\frac{n}{2}$ such iterations we have $\half(P) = 0$.
    
    For $i \in [\frac{n}{2}]$, let $X_i$ be independent geometric random variables, each with success probability $iq$, and let $X=\sum_{i=1}^{n/2} X_i$.
    Then $X$ stochastically dominates $T_0$.
    By applying Theorem~1.10.35 in \cite{Doerr20bookchapter} we have $\Pr[X \ge 2 q^{-1} \ln(n)] \le \frac{1}{n},$
    that is, with probability at least $1-\frac{1}{n}$ we have $T_0 <  2 q^{-1} \ln(n)$.
    
    Observe that a solution $x$ such that $\half(x)=0$ is never strictly dominated by any other solution.
    Thus we can employ the same proofs as used for Theorem~1 by only considering progress made on individuals $x$ such that $\half(x)=0$.
    This gives that the next
    \[\left(\frac{\ln(2)m'+2}{\ln(n)} + 16 \frac{{m'}^2+2m'}{n} + 2\right) q^{-1} \ln(n)\]
    iterations cover the complete Pareto front with probability at least $(1-\frac{1}{n})^2$.
    Combining the two bounds gives that the GSEMO solves \mcocz after $t$ iterations (and hence at most $t+1$ fitness evaluations) with probability at least $(1-\frac{1}{n})^3$, proving the first part of the statement. 
    We note that this applies for arbitrary starting configurations, as all that we assumed about the initial population was that it is non-empty. 
    
    We employ a simple restart argument to obtain an upper bound on the expected value of $F$.
    As each sequence of $t$ iterations fails to cover the Pareto front with probability at most $1-(1-\frac{1}{n})^3$ and due to the convergence of the geometric series we have 
    \[E[F] 
    \le 1 + \sum_{i=0}^{\infty} \left(1-\left(1-\tfrac{1}{n}\right)^3\right)^i t
    = \left(1-\tfrac{1}{n}\right)^{-3} t +1.\qedhere
    \]
\end{proof}

Comparing the bounds for \mOMM and \mcocz on bitstrings of the same length, we note that the bound for \mcocz is smaller than the one on \mOMM as $\SCOCZ \approx 2^{-m/2}\SOMM$.

\subsection{$m$\textsc{LeadingOnesTrailingZeros}}
\label{sec:mLOTZ}

In contrast to \mOMM and \mcocz, for \mlotz the probability to mutate a solution of a certain objective value into a desirable next solution where one of the values changed by $1$ is less depending on the objective value itself. In fact, this probability is at least $\frac{1}{n}$ for all objective values. With this, we can use the artificial continuation argument of~\cite{DoerrHK11}, that is, pretend that each iteration with this probability gives a progress towards the target (ignoring that the target might have been reached already). Since the probabilities of progress are asymptotically the same for each progress, we can analyze directly how each desired solution is found from the initial individual rather than first going for the corners as in the analysis of \mOMM. 
%Employing a similar strategy as in the proof of Lemma~\ref{lem:coverInwards}, but starting from the initial individual 
%instead of individuals in \corners
This gives the following Theorem~\ref{thm:upperBoundLOTZ}.

\begin{theorem}\label{thm:upperBoundLOTZ}
% Let $m'\in \N$ and  $m= 2m'$. 
% Consider the GSEMO optimizing \mLOTZ.
% Let $T$ denote the number of iterations until the population covers the complete Pareto front and let
% \begin{align*}
%     t = \max\left\{1, \frac{4 {m'}^2\ln\left(n'+1\right)+ 8 m'\ln(n)}{n}\right\} 2 e \SLOTZ \frac{n^2}{m'}.
% \end{align*} 
% Then $T \le \ceil{t}$ with high probability and $E[T] \le{(1-\frac{1}{n})^{-1}\ceil{t}}$.
Let $m',n'\in \N$, $n=m'n'$, and  $m= 2m'$.
Consider the GSEMO optimizing \mLOTZ.
Let $F$ denote the number of fitness evaluations until the population covers the complete Pareto front and let
\begin{align*}
    t\coloneqq \max\left\{1, \frac{4 {m'}^2\ln\left(\frac{n}{m'}+1\right)+ 8 m'\ln(n)}{n}\right\} 2 e \frac{n^2}{m'} \SLOTZ.
\end{align*} 
Then $F \le t+1$ with probability at least $1-\frac{1}{n}$ and $E[F] \le{(1-\frac{1}{n})^{-1}t}+1$.
\end{theorem}
\begin{proof}
Let $x_0$ be an individual in the initial population.
Consider any objective value $v = (a_1, n'-a_1, a_2, n'-a_2, \ldots, a_{m'}, n'-a_{m'})$ on the Pareto front and let $s_v$ be the corresponding bitstring.
We bound the time until $s_v$ is sampled by bounding the time until a marked individual $c$ becomes $s_v$.
Let initially $c = x_0$.
Whenever a Hamming neighbor $c'$ of $c$ is sampled, they differ in precisely one block. If the number of leading 1-bits or the number of trailing 0-bits in that block of $c'$ is closer to the respective value in $v$, we update $c$ to be $c'$.
Also, when $c$ is removed from the population, we replace it by any $c' \succeq c$ in the population.
The time until the population contains $s_v$ is dominated by the time until $c = s_v$.
For any fixed $i\in[m']$, let $T_i$ denote the number of iterations until the $i$th block of individual $c$ has exactly $a_{i}$ leading 1-bits and $n'-a_{i}$ trailing 0-bits.
We first bound $\Pr[T_i \le t]$.

Let the distance of $c$ to $s_v$ in block $i$ be defined as 
\[d_c = \max\set{0, a_i - \LOm_i(c)} + \max\set{0, n'-a_i - \TZm_i(c)},\]
where $\LOm_i(c)$ and $\TZm_i(c)$ denote the numbers of leading 1-bits and trailing 0-bits in the $i$th block of $c$, respectively.
Observe that $c = s_v$ if and only if $d_c = 0$, that $d_c \le n'$, and that if $c$ is replaced by an individual $c'$ with $c' \succeq c$, then we have $d_{c'} \le d_{c}$.
Thus, we have $c = s_v$ after at most $n'$ iterations that sample a Hamming neighbor of $c$ that flips either the first 0-bit or the last 1-bit in the $i$th block, depending on $c$ and $v$, and thereby decreases $d_c$ by at least 1.
Each iteration has probability at least $q$ to yield such an improvement.
% Thus, for the time $T_i$ until the $i$th block of $c$ is correct we have
    % $E[T_i] \le n' \cdot  \frac{1}{p}.$
For $j\in [\floor{t}]$, let $X_j$ be independent random variables, each with a Bernoulli distribution with success probability $q$.
Let $X=\sum_{j=1}^{\floor{t}} X_j$.
Then
$\Pr[T_i > t] \le \Pr\left[X \le n'-1\right].$
Since $t = \max\{1, \frac{4 {m'}^2\ln(n'+1)+ 8 m'\ln(n)}{n}\} 2 q^{-1} n'$,
% \[
% t = \max\left\{1, \frac{4 {m'}^2\ln\left(n'+1\right)+ 8 m'\ln(n)}{n}\right\} 2 q^{-1} n'
% \] 
we have $E[X]= q\floor{t} > 2 n'-1$ and thus
\begin{align*}
    \Pr\left[X \le n'-1\right] 
    \le  \Pr\left[X \le \tfrac{1}{2} E[X]\right].
\end{align*}
Applying a multiplicative Chernoff bound and exploiting that $E[X] = q\floor{t} \ge 8(m'\ln(n'+1)+2\ln(n))$ yields %, and noting that $\ln(m') < \ln(n)$ yields 
\begin{align*}
  \Pr[T_i > t] 
    &\le \exp\left(-\tfrac{1}{8}  E[X]\right)\\
 &   \le \exp\left(-\ln(n)-m'\ln\left(n'+1\right)-\ln(n)\right)\\
 &   \le \exp\left(-\ln(m')-m'\ln\left(n'+1\right)-\ln(n)\right).
\end{align*}
Using a union bound over all $m'$ blocks gives that any fixed objective value on the Pareto front is not sampled in $t$ iterations with probability at most 
\begin{align*}
    m'  \exp\left(-\ln(m')-m'\ln\left(n'+1\right)-\ln(n)\right)
    = \exp\left(-m'\ln\left(n'+1\right)-\ln(n)\right).
\end{align*}

Let $E$ denote the event that after $\floor{t}$ iterations there is still an objective value $v$ on the Pareto front such that the respective individual $c$ does not have the correct number of leading 1-bits and trailing 0-bits in any block.
By applying a union bound over all objective values we have
\begin{align*}
  \Pr[E] &\le \frontLOTZ \cdot \exp\left(-m'\ln\left(n'+1\right)-\ln(n)\right)
   = \exp(-\ln(n)) = \tfrac{1}{n}.
\end{align*}
by observing $\frontLOTZ = \exp(m'\ln(n'+1))$.
Thus, $E$ does \emph{not} happen with probability $1-\frac{1}{n}$.
We note that this applies for arbitrary starting configurations, as all that we assumed about the initial population was that it is non-empty. 

We employ a simple restart argument to obtain an upper bound on the expected value of $F$.
Each sequence of $\floor{t}$ iterations fails to cover the Pareto front with probability at most $\frac{1}{n}$.
Due to the convergence of the geometric series we have 
    \[E[F] 
    \le 1+\sum_{i=0}^{\infty} \left(\tfrac{1}{n}\right)^i \floor{t}
    \le \left(1-\tfrac{1}{n}\right)^{-1} t+1.\qedhere\]
\end{proof}
While, unlike our other bounds for the GSEMO, this result does not improve over the existing $O(n^{2m'+1})$ bound \cite{LaumannsTZ04}, we note that our bound applies to all choices for the numbers of objectives while the previous one assumed it to be constant.

\subsection{$m\textsc{OneJumpZeroJump}_k$}
\label{sec:mOJZJ}
The main result in this subsection is the following runtime guarantee for the GSEMO on the many-objective \ojzj problem.

\begin{theorem}
\label{thm:upperBoundJump}
Let $m'\in \N_{\ge 2}$, $n'\in \N_{\ge 4}$, $n=m'n'$, $m= 2m'$, and $k\in[2...\frac{n'}{2}]$. 
Consider the GSEMO optimizing \mJump.
Let $F$ denote the number of fitness evaluations until the population covers the complete Pareto front and let 
 \[t = \left(\ln(4)m'+\ln(n)+\ln(m')\right)\tfrac{13}{3} e n^k  \SOJZJ.\]
Then $F \le t+1$ with probability at least $(1-\frac{1}{n})^{-5}=1-\Theta(\frac{1}{n})$ and
% Further,
\[E[F] \le \left(1-\tfrac{1}{n}\right)^{-2} \left(\ln(4)m'+2\ln(m')\right)\tfrac{13}{3} e  n^k \SOJZJ +1.\]
% \[\left(\frac{1}{m'}\right) \]\simon{Ich will dieses 1/m' nicht, aber immer wenn ich es lösche wirft Overleaf für mich nicht erklärbare errors.}
\end{theorem}
% \begin{align*}
%     t =
%     % Phase 1 & 3
%     \bigg(16\frac{m'\ln(m')+{m'}^2\ln\left(n'-2k+3\right)+m'\ln(n)}{n}\\
%     + \left(\ln(2)\frac{m'}{\ln(n)}+2\right) \frac{\ln\left(n'-k\right)}{(1+k)}+ 2\bigg) e\SOJZJ n\\
%     % Phase 2
%     + \left(\ln(4)m'+\ln(n)+ \ln(m')\right) e \SOJZJ n^k
% \end{align*}
% \begin{align*}
%     t =
%     % Phase 1 
%     \left(\ln(2)\frac{m'}{\ln(n)}+2\right) \frac{1}{(1+k)q}\ln \left(n'-k\right)\\
%     % Phase 2
%     + \left(\frac{\ln(4)m'+\ln(n)}{\ln(m')}+1\right) \frac{1}{q_k} \ln(m')\\
%     % Phase 3
%     + \max\Big\{2\left(\frac{n'}{2}-k\right), 8\ln(m')+8m'\ln\left(n'-2k+3\right)\\
%     +8\ln(n)\Big\}\cdot \frac{2}{qn} m'.
% \end{align*}
Just as for the \mOMM benchmark, we obtain this runtime guarantee by explicitly analyzing the time until the corners are sampled. This once more is the dominant contribution to the runtime.
Unlike for the optimization for \mOMM, however, the \mJump benchmark requires a \emph{jump} over the $m'$ fitness valleys in order to sample a corner. We hence subdivide this first part into three phases.
First, we consider the time until a solution for an ``inner optimum'' is sampled, that is, a bitstring such that the number of 1-bits is at least $k$ and at most $n'-k$ in each block, see Lemma~\ref{lem:findInnerOptimum}. From this point on, we focus our analysis on objective values that are on the Pareto front and hence do not have to deal with solution values disappearing from the population by strict domination.
Starting with such an inner optimum, we consider the time until the population contains a solution for every ``extreme inner optimum'', which has precisely $k$ or $n'-k$ bits of value $1$ in each block, see Lemma~\ref{lem:walkToCliffs}. 
Next, we analyze how long it takes to cover all corner bitstrings starting from these solutions, see Lemma~\ref{lem:jumpToCorners}. This will be the dominant contribution to the runtime.
In the last phase, we consider the time to cover all remaining objective values, see Lemma~\ref{lem:coverInwardsJump}.

Formally, we define the set $\cliffs = \set{(a_1,\ldots,a_{m}) \mid a_i \in \set{2k, n'} \text{ for all } i\in[m]}$ of ``extreme inner optima'' to contain all objective values of individuals that in each block have either exactly $k$ bits of value~0 or exactly $k$ bits of value~1.
Further, we define the set $\cornersjump = \set{(a_1,\ldots,a_{m}) \mid a_i \in \set{k,2k, n', n'+k} \text{ for all } i\in [m']}$ to contain all objective values of individuals that in each block have either only 1-bits, only 0-bits, exactly $k$ bits of value~0, or exactly $k$ bits of value~1.

In order to be able to cross the valleys of low fitness, this benchmark requires $k$ simultaneous bit-flips. 
To this aim, we observe that the probability of creating any specific solution  $y$ that has a Hamming distance of $k$ to at least one solution $x$ in the current population is at least 
\begin{align*}
    \frac{1}{\SOJZJ} \cdot \left(\frac{1}{n}\right)^k \cdot \nootherflip[n-k] \ge \frac{1}{e n^k \SOJZJ} \eqqcolon q_k.
\end{align*}
We note that, due to some arguments in the proof of Lemma~\ref{lem:jumpToCorners}, the bounds we obtain for the \mJump benchmark are only applicable if $m' \ge 2$. For the case $m'=1$, we thus refer to previous results in the literature, which show that the expected number of iterations until the GSEMO solves the bi-objective $\maOJZJ_k$ problem is at most
 $e (\frac{3}{2}n^k + 2n \ln(\lceil\frac{n}{2}\rceil) +3)  S_{2,k}^{\maOJZJ}$
 \cite{ZhengD23ecj}.

 The following lemma bounds the time until the population contains an inner optimum by considering the creation of any Pareto optimum as an intermediate step.
 
\begin{lemma}\label{lem:findInnerOptimum} 
Let $m'\in \N$, $n'\in \N_{\ge 4}$, $n=m'n'$, $m= 2m'$, and $k\in[2...\frac{n'}{2}]$.  Consider the GSEMO optimizing \mJump. 
Let $T$ denote the number of iterations until the population contains a bitstring with at least $k$ and at most $n'-k$ bits of value $1$ in each block and let $t_0$ be
 \[\max\set{2km',8\ln(n)}\tfrac{1}{(n'-k)q}+\max\set{2m',8\ln(n)}\tbinom{\frac{n}{m'}}{k}^{-1} \tfrac{1}{q_k}.\]         
Then $T \le t_0$ with probability at least $(1-\frac{1}{n})^2.$
\end{lemma}
\begin{proof}
We first show that after 
\[s_1 \coloneqq \max\set{2km',8\ln(n)}\tfrac{1}{(n'-k)q}\]
iterations the population contains a solution on the Pareto front with probability at least $1-\frac{1}{n}$.
Suppose the population does not already contain such a solution as otherwise the bound trivially holds.
For every solution $x$ and block $i\in[m']$, let 
\[L_i(x) = \begin{cases}
    k-|x^i|_1,\hphantom{aaaaaa} \text{\quad if } 0 < |x^i|_1 < k;\\
    |x^i|_1 -(n'-k), \text{\quad if } n'-k < |x^i|_1 < n';\\
    0, \hphantom{aaaaaaaaaaaa} \text{\quad else.}\\
\end{cases}\]
Let $L(x) = \sum_{i\in[m']}L_i(x)$ and, for a population $P$, let $L(P) = \min_{x\in P}L(x)$. 

Then every generation decreases or preserves $L(P)$ as follows. Let $x\in P$ such that $L(x) = L(P)$.
By the definition of the GSEMO, the next population contains a solution $y\succeq x$.
Consider any block $i\in [m']$. 
Note that $y \succeq x$ and $L_i(x) = 0$ imply $L_i(y) = L_i(x) = 0$.
If $0 < |x^i|_1 < k$, then $y \succeq x$ implies that $|x^i|_1 \le |y^i|_1 \le n'-k$ or $|y^i|_1 = n'$, so $L_i(y) \le L_i(x)$. The symmetric case of $n'-k < |x^i|_1 < n'$ follows analogously. Thus $L(y) \le L(x)$.

As long as $L(P) > 0$, each iteration has a chance of at least $p\coloneqq (n'-k)q$ to decrease $L(P)$ as follows. Consider an individual $x$ with $L(x)=L(P)$ and let $i$ be such that $L_i(x) > 0$. There are at least $n'-k+1$ bits in this block that we can flip to decrease $L_i(x)$. Selecting $x$ for mutation and creating a Hamming neighbor that precisely flips any of these desirable bits has a chance of at least $p$.

Note that $L(P) \le km'$. For $j \in [s_1]$, let $X_j$ be independent random variables, each with a Bernoulli distribution with success probability $p$.
Let $X=\sum_{j=1}^{s_1} X_j$.
Then, by stochastic domination, the probability that after $s_1$ iterations we still have $L(P)>0$ is at most $\Pr[X < km']$.
By observing $E[X]= ps_1 \ge 2km'$ we have
\begin{align*}
    \Pr[X < km'] 
    \le  \Pr\left[X \le \tfrac{1}{2} E[X]\right].
\end{align*}
Then, by applying a multiplicative Chernoff bound, the probability to sample a Pareto optimum within $s_1$ iterations is at least
\begin{align*}
    1-\exp\left(-\tfrac{1}{8}  E[X]\right)
    \ge 1- \exp\left(-\ln(n)\right) = 1-\tfrac{1}{n}.
\end{align*}

Next, starting with any Pareto optimum in the population, we show that after 
\[s_2 \coloneqq \max\set{2m',8\ln(n)}\tbinom{n'}{k}^{-1} \tfrac{1}{q_k}\]
iterations the population contains an inner optimum with probability at least $1-\frac{1}{n}$.
Let $x$ be any Pareto optimal solution in the population. 
Each block in $x$ has either $0$, $n'$, or between $k$ and $n'-k$ bits of value 1.
For every  Pareto optimal solution $x$ let $L'(x)$ denote the number of blocks with precisely 0 or $n'$ bits of value 1. 
For a population $P$, let $L'(P)$ the minimum value of $L'(x)$ among all Pareto optimal solution $x\in P$.
As there is at least one Pareto optimal solution, we have $L'(P) \le m'$. 
The GSEMO preserves each Pareto optimal objective value represented in the population and $L'(x)$ can be inferred from the objective value of $x$, hence no iteration increases $L'(P)$.

As long as $L'(P) > 0$, each iteration has a chance of at least $p'\coloneqq \binom{n'}{k} q_k$ to decrease $L(P)$ as follows. Consider any Pareto optimal solution $x\in P$ with $L(x)=L(P)$ and consider any block $i$ with just 1-bits or just 0-bits. There are at least $\binom{n'}{k}$ combinations of $k$ bits in this block that decrease $L(x)$ upon a simultaneous flip. Selecting $x$ for mutation and creating a Hamming neighbor that precisely flips such a set of desirable bits has a chance of at least $p$. 
Note that $L'(P) \le m'$. For $j \in [s_2]$, let $X'_j$ be independent random variables, each with a Bernoulli distribution with success probability $p'$.
Let $X'=\sum_{j=1}^{s_2} X'_j$.
Then, by stochastic domination, the probability that after $s_2$ iterations we still have $L'(P)>0$ is at most $\Pr[X' < m']$.
By observing $E[X']= p's_2 \ge 2m'$ we have
\begin{align*}
    \Pr[X' < m'] 
    \le  \Pr\left[X' \le \tfrac{1}{2} E[X']\right].
\end{align*}
Then, by applying a multiplicative Chernoff bound, the probability to sample a Pareto optimum within $s_2$ iterations is at least
\begin{align*}
    1-\exp\left(-\tfrac{1}{8}  E[X']\right)
    \ge 1- \exp\left(-\ln(n)\right) = 1-\tfrac{1}{n}.
\end{align*}
Thus, with probability at least $(1-\frac{1}{n})^2$ the population contains an inner optimum after at most $s_1 + s_2 = t_0$ iterations.
\end{proof}
We note that in the above analysis we considered the progress made in different blocks to happen sequentially. Using similar techniques as in our other proofs to analyze the progress made in each block in parallel could improve the runtime for this phase by a factor of around $\frac{m'}{\ln(m')}$. However, as the runtime of the optimization for \mJump is in any case dominated by the other phases, we here gave the more direct proof with a slightly worse runtime guarantee.

The analysis of the phase of sampling all ``extreme inner optima'' $\cliffs$ is roughly similar to the one of the first phase on \mOMM in Lemma~\ref{lem:walkToCorners}, that is, we analyze separately the progress in each block towards the target solution value. 
% However, extra arguments are required as \cliffs for \mJump can be approached from two directions (the gap region and the slope region of the jump function) and as for the \mJump problem, we can have strict domination between solutions. Both require a careful definition of the stochastic process measuring the progress. 

Recalling the notation established in Section~\ref{sec:mOMM} for the analysis of \mOMM, we refer to a bitstring with $a_i$ bits of value one in block $i$ for each $i\in[m']$ as an \gabstring.
\begin{lemma}\label{lem:walkToCliffs} 
Let $m'\in \N$, $n'\in \N_{\ge 4}$, $n=m'n'$, $m= 2m'$, and $k\in[2...\frac{n'}{2}]$.  Consider the GSEMO optimizing \mJump starting with a population that contains at least one solution for an inner optimum. 
Let $T$ denote the number of iterations until the population covers \cliffs and let 
 \[t_1 = \left(\ln(2)\frac{m'}{\ln(n)}+2\right) \frac{1}{q}\ln \left(\frac{n}{m'}-2k\right).\]
 % \[t = \left(\ln(2)\frac{m'}{\ln(n)}+2\right)e \SOJZJ (n-k) \ln(n-k).\]

 %                                          
Then $T \le t_1$ with probability at least $1-\frac{1}{n}.$
\end{lemma}
\begin{proof}
Observe that by the symmetry of the problem and the operators, the time to cover any fixed objective value in \cliffs is the same for all elements in \cliffs.
Thus, we first give a tail bound for the time until the population contains a $(k,\ldots,k)$-bitstring and note that its objective value is in \cliffs.

Let $x_0$ be a bitstring for an inner optimum in the initial population. 
We bound the time until a $(k,\ldots,k)$-bitstring is sampled by bounding the time until a marked individual $c$ becomes a $(k,\ldots,k)$-bitstring.
Let initially $c = x_0$, that is, $c$ has at least $k$ bits of value $1$ in each block.
Whenever a Hamming neighbor $c'$ of $c$ is sampled, they differ in precisely one block. If the number 1-bits in that block of $c'$ smaller than in that block of $c$ and at least $k$, we update $c$ to be $c'$ and note that $c'$ is an inner optimum as well. 
As $c$ is always on the Pareto front, whenever the individual $c$ is removed from the population, there is $c'$ in the population with the same objective value and we replace $c$ by $c'$ without losing any progress.
The time until the population contains a $(k,\ldots,k)$-bitstring is at most the time until $c$ is a $(k,\ldots,k)$-bitstring.

Let $|c^i|_1$ denote the number of 1-bits in the $i$th block of $c$ and define $d_{c,i} = |c^i|_1 - k$.
Note that once $d_{c,i} = 0$ for some $i \in [m']$, the objective values of the $i$th block of $c$ will change no more and if $d_{c,i} = 0$ for all $i \in [m']$ then $c$ is a $(k,\ldots,k)$-bitstring.
We first give a tail bound on the time until $d_{c,i} = 0$ for any fixed $i \in [m']$.

Note that $d_{c,i}$ is non-increasing as every update of $c$ either decreases or preserves the number of 1-bits in each block. Every iteration has a chance of at least 
\begin{align*}
    |c^i|_1 q > d_{c,i} q 
    = \left(n'-2k\right)q \cdot \frac{d_{c,i}}{n'-2k}
    % \ge \frac{d_{c,i}}{e \SOJZJ (n-k)}
    \eqqcolon p_{d_{c,i}}
\end{align*}
to decrease $d_{c,i}$ by sampling a Hamming neighbor $c'$ of $c$ that flips precisely one 1-bit in the $i$th block.
As $|x_0^i|_1 \le n'-k$, we have that after at most $n'-2k$ such iterations we have $d_c = 0$.
For $j \in [n'-2k]$, let $X_j$ be independent geometric random variables, each with success probability $p_j$, and let
    $X = %\sum_{d_c=1}^{n-k} X_j\ge
    \sum_{j=1}^{n/m'-2k} X_j$.
Let $T_C$ denote the number of iterations until $c$ is a $(k,\ldots,k)$-bitstring. 
Then $X$ stochastically dominates $T_C$, and thus a tail bound for $X$ also applies to $T_C$.
By once more applying Theorem~1.10.35 in \cite{Doerr20bookchapter} we have
    \[\Pr\left[X\ge (1+\delta) \frac{1}{q} \ln \left(n'-2k\right)\right] \le n^{-\delta}\]
    % \[\Pr[X\ge (1+\delta)e \SOJZJ (n-k) \ln (n-k)] \le n^{-\delta}\]
 for all $\delta \ge 0$.
 For $\delta = \ln(2)\frac{m'}{\ln(n)}+1$ we obtain that no $(k,\ldots,k)$-bitstring is sampled after $t_1$ iterations with probability at most $n^{-\ln(2)\frac{m'}{\ln(n)}-1}$.

Let $E$ denote the event that after $t_1$ iterations there is still an objective value in \cliffs such that the population does not contain a corresponding individual.
By applying a union bound we have
\begin{align*}
  \Pr[E] \le |\cliffs| \cdot n^{-\ln(2)\frac{m'}{\ln(n)}- 1} = \frac{1}{n}
\end{align*}
by observing $|\cliffs| = 2^{m'} = n^{\ln(2)m'/\ln(n)}$.
% We note that this applies for arbitrary starting configurations, as all that we assumed about the initial population was that it is non-empty. 
\end{proof}

We now turn to the most interesting phase, where we find a solution for each objective value in \cornersjump when starting with a population that contains a solution for each value in \cliffs (that is, the result of the second phase). This third phase dominates the optimization time since here the valleys of low objective values have to be crossed in the difficult direction, requiring $k$ simultaneous specific bit-flips each. 

\begin{lemma}\label{lem:jumpToCorners}
    Let $m'\in \N_{\ge 2}$, $n'\in \N_{\ge 4}$, $n=m'n'$, $m= 2m'$, and $k\in[2...\frac{n'}{2}]$. 
    Consider the GSEMO optimizing \mJump starting with a population that contains at least one individual of each objective value in \cliffs.
    Let $T$ denote the number of iterations until the population covers \cornersjump and let 
          \[t_2 = \left(\frac{\ln(4)m'+\ln(n)}{\ln(m')}+1\right) \frac{1}{q_k} \ln(m').\]
    Then $T \le t_2$ with probability at least $1-\frac{1}{n}$.
    Further, 
    \[E[T] \le \left(1-\frac{1}{m'}\right)^{-1} \left(\frac{\ln(4)m'}{\ln(m')}+2\right) \frac{1}{q_k} \ln(m').\]
\end{lemma}
\begin{proof}
Consider any objective value $v\in \cornersjump$ and let $a_1,\ldots,a_{m'}$ be such that every \gabstring has objective value $v$.
Let $c_0$ be any individual in the population with objective value in \cliffs that is closest to any \gabstring.
We bound the time until an \gabstring is sampled by bounding the time until a marked individual $c$ becomes an \gabstring.
Let initially $c = c_0$.
Whenever an individual $c'$ is sampled that has the same bit values as $c$ except it has flipped the remaining $k$ bits in any block, we update $c$ to be $c'$ and call this \emph{progress}.
Note that the objective value of $c$ is always on the Pareto front. Hence, when $c$ is removed from the population, there is $c'$ in the population with the same objective value and we replace $c$ by $c'$ without losing any progress.
The time $T_v$ until the population contains an \gabstring is at most the time until $c$ is an \gabstring.
Suppose $c$ differs from being an \gabstring in $i$ blocks.
Each iteration has a probability of at least $i q_k$ to yield progress. 
Note that after at most $m'$ iterations of progress, $c$ is an \gabstring.

Let $X_i$ be a independent geometric random variable for all $i \in [m']$, each with success probability $i q_k$, and let $X=\sum_{i=1}^{m'} X_i$.
Then $X$ stochastically dominates $T_v$, and thus a tail bound for $X$ also applies to $T_v$. 
By once more applying Theorem~1.10.35 in \cite{Doerr20bookchapter}, we have
\begin{align}
    \Pr\left[X\ge (1+\delta) \tfrac{1}{q_k} \ln(m') \right] \le {m'}^{-\delta}.
    \label{eq:cliffJump}
\end{align}
By choosing $\delta=\frac{\ln(4)m'+\ln(n)}{\ln(m')}$, the probability to not sample an \gabstring in the next $t_2$ iterations is at most 
\begin{align*}
 {m'}^{-\ln(4)\frac{m'}{\ln(m')}-\frac{\ln(n)}{\ln(m')}}
  = n^{-\ln(4)\frac{m'}{\ln(n)}-1}.
\end{align*}

Let $E$ denote the event that after $t$ iterations there is still an objective value $v\in \cornersjump$ such that the population does not contain a corresponding individual. 
By applying a union bound, we have
\begin{align*}
    \Pr[E] \le |\cornersjump| \cdot n^{-\ln(4)\frac{m'}{\ln(n)}-1} = \frac{1}{n}
\end{align*}
by observing $|\cornersjump| = 4^{m'} = n^{\ln(4)m'/\ln(n)}$.

For the expected value, consider Equation~(\ref{eq:cliffJump}) with $\delta = \ln(4)\frac{m'}{\ln(m')} +1$ and recall that $m' \ge 2$.
This gives that the runtime is at most 
\[t' \coloneqq \left(\frac{\ln(4)m'}{\ln(m')}+2\right) \frac{1}{q_k} \ln(m')\]
iterations with probability at least $m'^{-\ln(4)\frac{m'}{\ln(m')} -1}$.
Let $E'$ denote the event that after $t'$ iterations there is still an objective value $v\in \cornersjump$ such that the population does not contain a corresponding individual. 
By applying a union bound and observing $|\cornersjump| = {m'}^{\frac{\ln(4)m'}{\ln(m')}}$ we have
\begin{align*}
    \Pr[E'] \le |\cornersjump| \cdot {m'}^{-\ln(4)\frac{m'}{\ln(m')}-1} = \frac{1}{m'}.
\end{align*}
We employ a simple restart argument.
Each sequence of $t'$ iterations fails to cover the Pareto front with probability at most $\frac{1}{m'}$.
Due to the convergence of the geometric series we have 
    \[E[T] 
    \le \sum_{i=0}^{\infty}  \left(\frac{1}{m'}\right)^i t'
    = \left(1-\frac{1}{m'}\right)^{-1} t'.
    \qedhere
    \]
\end{proof}

We observe that $t_2$ has an additional summand of $\frac{\ln(n)}{\ln(m')}$ in the parenthesis that is not present in the bound of the expected value.
This additional summand is necessary to obtain the tail bound, as can be seen best when considering a constant value for~$m'$. 
Then, a constant number of events (jumps over the fitness valley) have to happen, each with an estimated waiting time of $e \SOJZJ n^k$. In order to obtain that all these events happen with high probability, that is, probability $1 - O(\frac{1}{n})$, in a given time, this time has to depend on $n$ to some degree, and this is the additional $\frac{\ln(n)}{\ln(m')}$ term.
% For the expected value this does not matter, as the expected time until all events happen is bounded by the sum over the (constantly many) individual expected waiting times.

When regarding the proof, we note that we have directly constructed a solution for each value in \cliffs, different from the proof for \omm, where we have first constructed all corners \corners and from these constructed the remaining solutions. We could have done the same here, but it would not have given a better runtime estimate. The reason is that the number of corners, $2^{m'}$, is not that much smaller than the number of values in \cliffs, which is $4^{m'}$. Hence the union bound over the larger numbers of values in \cliffs is, asymptotically, not more costly. Note that in our setting with exponential tails, the number of events a union bound is taken over influences the final result only logarithmically.

We now consider the part that accounts for generating solutions for the remaining part of the Pareto front. We note that for every missing solution, the objective values for each block lie either in the inner Pareto front or beyond one of the two fitness valleys. Thanks to the third phase, for each possible combination there is already a solution in the current population. Thus no more jumps or walks through the fitness valley are required and all that remains is a walk on the respective part of the Pareto front to generate the exact objective value. 
Apart from these changed preconditions, the analysis is comparable to the process of covering the remaining Pareto front of \mOMM in Lemma~3.
\begin{lemma}\label{lem:coverInwardsJump}
Let $m'\in \N_{\ge 2}$, $n'\in \N_{\ge 4}$, $n=m'n'$, $m= 2m'$, and $k\in[2...\frac{n'}{2}]$. 
Consider the GSEMO optimizing \mJump starting with a population that contains at least one individual of each objective value in \cornersjump.
Let $T$ denote the number of iterations until the population covers the complete Pareto front and let $t_3$ be
\begin{align*}
    \max\Big\{2\left(\frac{n}{2m'}-k\right), 8\ln(m')+8m'\ln\left(\frac{n}{m'}-2k+3\right)
    +8\ln(n)\Big\}\cdot \frac{2m'}{qn}.
\end{align*}
Then $T \le t_3$ with probability at least $1-\frac{1}{n}$.
\end{lemma}
\begin{proof}
Consider any objective value $v$ on the Pareto front that is not in \cornersjump. Let $a_1,\ldots,a_{m'}$ be such that every \gabstring has objective value $v$.
Let $c_0$ be any individual in the population with objective value in \cornersjump that is closest to any \gabstring.
We bound the time until an \gabstring is sampled by bounding the time until a marked individual $c$ becomes an \gabstring.
Let initially $c = c_0$.
Whenever a Hamming neighbor $c'$ of $c$ is sampled that is closer to any \gabstring than $c$, we update $c$ to be $c'$.
Note that the objective value of $c$ is always on the Pareto front. Hence, when $c$ is removed from the population, there is $c'$ in the population with the same objective value and we replace $c$ by $c'$ without losing any progress.
The time until the population contains an \gabstring is at most the time until $c$ is an \gabstring.

Due to the symmetry of the problem, we can without loss of generality for all $i\in [m']$ assume $a_i \le \frac{n'}{2}$ and $c_0$ to have no bits of value~1 if $a_i = 0$ and $k$ bits of value~1, otherwise.
We first give a tail bound on the time until $c$ has $a_i$ bits of value~1 in the $i$th block for any fixed $1\le i \le m'$.
The probability of increasing the number of 1-bits in block $i$ in an iteration by sampling a Hamming neighbor of $c$ that differs from $c$ in one of the at least $\frac{n'}{2}$ bits of value~0 in the $i$th block of $c$ is at~least
\begin{align*}
    p \coloneqq \frac{n'}{2}\cdot q.
\end{align*}
After at most $a_i - k \le n'-k$ such iterations, $c$ contains the correct number of 1-bits in the $i$th block. 
Thus, for the time $T_i$ until the $i$th block of $c$ is correct we have $E[T_i]\le (a_i-k)\cdot \frac{1}{p}$.
For $1\le j \le \floor{t_3}$, let $X_j$ be independent random variables, each with a Bernoulli distribution with success probability $p$.
Let $X=\sum_{j=1}^{\floor{t_3}} X_j$.
Then $\Pr[T_i > t_3] \le \Pr[X \le a_i-k - 1]$.
By observing $E[X]= p\floor{t_3} > 2(\frac{n'}{2}-k)-1$ we have  %t >= 4 S (n - mk)
\begin{align*}
    \Pr[X \le a_i-k-1] 
   & \le \Pr\left[X \le \tfrac{n'}{2}-k-1\right] 
   \le  \Pr\left[X \le \tfrac{1}{2} E[X]\right].
\end{align*}
Applying a multiplicative Chernoff bound yields
\begin{align*}
    \Pr&[T_i > t_3] 
    \le \exp\left(-\tfrac{1}{8}  E[X]\right)
    \le \exp\left(-\ln(m')-m'\ln\left(n'-2k+3\right)-\ln(n)\right).
\end{align*}
Using a union bound over all blocks gives that any fixed objective value on the Pareto front is not sampled in $\floor{t_3}$ iterations with probability at most 
\begin{align*}
    m' \cdot &\exp\left(-\ln(m')-m'\ln\left(n'-2k+3\right)-\ln(n)\right) \\
    &= \exp\left(-m'\ln\left(n'-2k+3\right)-\ln(n)\right).
\end{align*}

Let $E$ denote the event that after $\floor{t_3}$ iterations there is still an objective value in the Pareto front that does not have a respective individual in the population.
By applying a union bound we have
\begin{align*}
  \Pr[E] &\le \frontOJZJ \cdot \exp\left(-m'\ln\left(n'-2k+3\right)-\ln(n)\right)
    =  \exp(-\ln(n)) = \tfrac{1}{n}
\end{align*}
by observing
\begin{align*}
    \frontOJZJ = \left(n'-2k+3\right)^{m'} = \exp\left(m'\ln\left(n'-2k+3\right)\right). 
    \end{align*}
\end{proof}

From the partial results just proven, we can now derive our main result for the many-objective \ojzj problem.

\begin{proof}[Proof of Theorem~\ref{thm:upperBoundJump}]
Let $t_0, t_1, t_2$ and $t_3$ be as in Lemmas~\ref{lem:findInnerOptimum} to \ref{lem:coverInwardsJump}. For the second phase, $t_1$ is not defined if $k = \frac{n'}{2}$. We note that in this case $\cliffs$ consists of just one inner optimum with $\frac{n'}{2}$ bits of value $1$ in each block. As this is the only inner optimum, it is already covered as result of the previous phase. Hence let $t_1 = 0$ in this case. 

Let $t' = t_0 + t_1 + t_2 +t_3$.
Let $T$ denote the number of iterations until the population covers the complete Pareto front.
The combination of the lemmas gives that $T \le t'$ with probability at least $(1-\frac{1}{n})^5$.

Observe that for $q = \frac{1}{en\SOJZJ }$ and $q_k = \frac{1}{en^k\SOJZJ}$ we have that $t_2 \ge \max\set{t_1, t_3}$ and $\frac{4}{3}t_2 \ge t_0$ as $\binom{n'}{k} \ge 6$ with $2 \le k \le \frac{n'}{2}$. 
Thus, $t := \frac{13}{3}t_2 \ge t'$. As each iteration as well as the initialization takes one fitness evaluation, we have $F \le t+1$ with probability at least $(1-\frac{1}{n})^5$.

For the expected value, we employ the same simple restart argument as before, separately for each phase, and obtain that the expected number of iterations for the first, second and forth phase are at most $(1-\frac{1}{n})^{-2}t_0$, $(1-\frac{1}{n})^{-1}t_1$ and $(1-\frac{1}{n})^{-1} t_3$, respectively.
% \[\left(1-\frac{1}{n}\right)^{-1} \left(\ln(2)\frac{m'}{\ln(n)}+2\right)e \SOJZJ (n-k) \ln(n-k)  \]
% and 
% \begin{align*}\lceil \left(1-\frac{1}{n}\right)^{-1}\max\Big\{2\left(\frac{n'}{2}-k\right),
% 8\ln(m')\\ + 8m'\ln\left(n'-2k+3\right) +8\ln(n)\Big\}\cdot 2 e m' \SOJZJ \rceil,
% \end{align*}
% respectively.
This holds as the arguments in the lemmas can be applied repeatedly, as all they assume about the initial population was that it is non-empty (for the first phase) and covers objective values generated in the previous phases (for the other phases).
The expected time of the second phase is bounded from above by
\begin{align*}
    \left(1-\tfrac{1}{m'}\right)^{-1} \left({\ln(4)m'}+2\ln(m')\right) e n^k \SOJZJ,
\end{align*}
see Lemma~\ref{lem:jumpToCorners}. By combining the bounds and adding $1$ for the initial evaluation we obtain
\[E[F] \le \left(1-\tfrac{1}{n}\right)^{-2} \left(\ln(4)m'+2\ln(m')\right)\tfrac{13}{3} e  n^k \SOJZJ +1.\qedhere\]
% \qed
\end{proof}

\section{Extension of our Results to the SEMO, SMS-EMOA, an NSGA-II variant, NSGA-III, and SPEA2}
\label{sec:otherAlgos}

We started our mathematical runtime analysis of many-objective MOEAs with an analysis of the GSEMO, the most prominent MOEA in theoretical works. We now show that our methods also apply to many other classic MOEAs, including the prominent \SMS, \NSGAthree, and SPEA2. To this aim, we show the following meta-theorem.

\begin{theorem}[Meta-theorem]
    \label{thm:meta-theorem}
    Consider a MOEA $\calA$ that starts with a population $P_0 \neq \emptyset$ and in each iteration $t$ computes $P_t$ from $P_{t-1}$. Consider further a multi-objective optimization problem~$\calB$ over bitstrings of length $n$. Let $m'\in \N$, and $m=2m'$ be the number of objectives in $\calB$. Assume that there is $q\in \Q_{>0}$ such that the following properties are satisfied when $\calA$ optimizes $\calB$.
   
    \begin{enumerate}
        \item Monotonicity property: If in some iteration~$t \in \Z_{\ge 0}$, including the initialization, a solution $x$ is generated, then in all future iterations $t' \ge t$ the main population $P_{t'}$ of the algorithm contains a solution $x'$ such that $x' \succeq x$. 
        %For each iteration $t\in \N_{\ge 0}$ and each solution $x$ in the population $P_t$, there is $x'\in P_{t'}$ with $x'\succeq x$ for all $t' \ge t$,
        % \item For each iteration $t\in \Z_{\ge 1}$, consider a solution  and each $x\in P_{t-1}$, the chance to select $x$ for mutation in iteration $t$ is at least~$s$.
        % \item The mutation of a bitstring $x$ has a chance of at least $p_1$ to produce any specific bitstring $x'$ with Hamming distance one from $x$.
        \item For each iteration $t\in \Z_{\ge 1}$ and each solution $x'$ that is a Hamming neighbor of some $x\in P_{t-1}$, the chance to generate $x'$ in iteration $t$ is at least $q$.
    \end{enumerate}
    Let $T$ denote the number of iterations until the population of $\calA$ covers the complete Pareto front of $\mathcal{B}$.
    \begin{itemize}
    \itemsep1em 
        \item If $\mathcal{B} = \mOMM$ and 
        \[
        t\coloneqq \left(\tfrac{\ln(2)m'+2}{\ln(n)} + 16 \tfrac{{m'}^2+2m'}{n} + 2\right) \tfrac{1}{q} \ln(n),
        \]
        then $T \le t$ with probability at least $(1-\frac{1}{n})^2 =1-\Theta(\tfrac{1}{n})$ and $E[T] \le (1-\frac{1}{n})^{-2} t$.

        \item If $\mathcal{B} =\ \mcocz$ and
        \[t \coloneqq\left(\tfrac{\ln(2)m'+2}{\ln(n)} + 16 \tfrac{{m'}^2+2m'}{n} + 4\right) \frac{1}{q} \ln(n),\]
        then $T \le t$ with probability at least $(1-\tfrac{1}{n})^3 =1-\Theta(\tfrac{1}{n})$ and $E[T] \le (1-\tfrac{1}{n})^{-3} t$.
        \item If $\mathcal{B} =\ \mlotz$ and
        \[t \coloneqq \max\left\{1, \tfrac{4 {m'}^2\ln\left(\frac{n}{m'}+1\right)+ 8 m'\ln(n)}{n}\right\} \tfrac{2}{q} n',\]
        then $T \le t$ with probability at least $1-\tfrac{1}{n}$ and $E[T] \le (1-\tfrac{1}{n})^{-1}t$.
        \item {Let $\mathcal{B} = \mJump$ with $m' \ge 2$,  $k\in[2 .. \frac{n}{2m'}]$ and $q_k \in \Q_{>0}$. Assume that for each iteration $t\in \Z_{\ge 1}$ and each solution $x'$ for which there is a solution $x\in P_{t-1}$ with Hamming distance exactly~$k$, the probability to generate $x'$ in iteration $t$ is at least $q_k$.} Let 
          \begin{multline*}
               t \coloneqq \big(14m'
              +  \tfrac{16{m'}^2}{n}\ln\left(\tfrac{n}{m'}-2k+3\right) 
            +6\ln(n)  +2\big)\tfrac{1}{q}\\
              + \tfrac{7}{3}\big(\ln(4)m'+\ln(n)+ \ln(m')\big) \tfrac{1}{q_k}.\hspace{2.75cm}
        \end{multline*}
       Then $T \le t$ with probability at least $(1-\frac{1}{n})^5 =1-\Theta(\frac{1}{n})$ and
       $E[T] \le (1-\tfrac{1}{n})^{-2}\Big(14m' +  \tfrac{16{m'}^2}{n}\ln\left(\frac{n}{m'}-2k+3\right) 
            +6\ln(n) +2\Big)\tfrac{1}{q}
             + (1-\tfrac{1}{m'})^{-1}\tfrac{7}{3}(\ln(4)m'+2\ln(m')) \tfrac{1}{q_k}$.
    \end{itemize}
\end{theorem}
\begin{proof}
    Note that all proofs for the runtime bounds in Section~\ref{sec:analysis_gsemo} did not leverage any specific properties of the GSEMO except for (i)~its monotonicity property and (ii)~that the probability that mutating a solution $x$ generates any specific bit string in Hamming distance one (or $k$ for \mJump) with probability at least $q$ (or $q_k$).
    In particular, note that the above proofs do not assume any particular value of $q$ and $q_k$ except when formulating the final bound.
{Thus, the proofs used for the GSEMO hold for the MOEA $\calA$ discussed in this theorem.
    The bounds in Theorems~\ref{thm:upperBoundOMM},~\ref{thm:upperBoundCOCZ},~and~\ref{thm:upperBoundLOTZ} before plugging in the value of $q$ and $q_k$ immediately give the first three stated guarantees,} where we note that we here give the bounds with respect to the number of iterations rather than the number of fitness evaluations as we cannot, like for the GSEMO, assume that $\calA$ uses one evaluation per iteration. 
    
    For \mJump, consider $t_0, t_1, t_2, t_3$ as in Lemmas~\ref{lem:findInnerOptimum} to \ref{lem:coverInwardsJump} and, as discussed in the proof of Theorem~\ref{thm:upperBoundJump}, let $t_1 = 0$ if $k = \frac{n'}{2}$.
    % \changed{
    These lemmas extend to our general setting and, as in the proof of Theorem~\ref{thm:upperBoundJump}, we have that $T \le t_0 + t_1+t_2+ t_3$ with probability at least $(1-\frac{1}{n})^5$. With $n'=\frac{n}{m'}$ and
    \begin{align*}
  w_1 \coloneqq \Bigg(&\tfrac{\max\set{2km',8\ln(n)}}{n'-k}+\left(\ln(2)\tfrac{m'}{\ln(n)}+2\right)
  \ln \left(n'-2k\right)+ \\
   &\max\big\{\tfrac{n'}{2}-k, 4\ln(m')+4m'\ln\left(n'-2k+3\right)
    +4\ln(n)\big\}
    \cdot \tfrac{4m'}{n}\Bigg)\tfrac{1}{q}
    \end{align*}
    and $w_2 \coloneqq \tfrac{7}{3}(\ln(4)m'+\ln(n)+ \ln(m')) \tfrac{1}{q_k}$ and using similar estimates as in the proof of Theorem~\ref{thm:upperBoundJump}, we obtain $t_0 + t_1+t_2+t_3 \le w_1 + w_2$.

Usually we have $q_k = \Omega(qn)$ and thus $w_2$ dominates the running time. To improve readability, we hence give the slightly less tight estimate of 
\begin{align*}
    w_1 \le \Big(14m'
  +  \tfrac{16{m'}^2}{n}\ln\left(n'-2k+3\right) 
+6\ln(n)  +2\Big)\tfrac{1}{q},
\end{align*}
where we use the facts that 
\begin{align*}
   \frac{\max\set{2km', 8\ln(n)}}{n'-k} \le \frac{2km'+ 8\ln(n)}{k} \le  2m' + 4\ln(n),
\end{align*}
that
\begin{align*}
   \left(\ln(2)\tfrac{m'}{\ln(n)}+2\right) \ln \left(n'-2k\right)
   \le \ln(2)m' + 2\ln(n),
\end{align*}
that, as $n\ge 4$,
\begin{align*}
   (4\ln(m')+4\ln(n))\cdot \tfrac{4m'}{n} \le 32m' \tfrac{\ln(n)}{n} \le 12m',
\end{align*}
and that $(\tfrac{n'}{2}-k)\cdot \tfrac{4m'}{n} \le 2$.
With this estimate, we note that $t \ge w_1 + w_2$, where $t$ is as given in the theorem statement for \mJump, and thus $T \le t$ with probability at least $(1-\frac{1}{n})^5$.

For the expected optimization time we have
\begin{align*}
E[T] &\le (1-\tfrac{1}{n})^{-2} (t_0+t_1+t_3)
+ \left(1-\tfrac{1}{m'}\right)^{-1} \left(\ln(4)m'+2\ln(m')\right) \tfrac{1}{q_k}\\
&\le {(1-\tfrac{1}{n})^{-2} \Big(14m'
              +  \tfrac{16{m'}^2}{n}\ln\left(n'-2k+3\right) 
            +6\ln(n)  +2\Big)\tfrac{1}{q}}\\
             & {\hphantom{aaa}+ (1-\tfrac{1}{m'})^{-1} \tfrac{7}{3}(\ln(4)m'+2\ln(m')) \tfrac{1}{q_k}}
\end{align*}
by reusing the arguments given in the proof of Theorem~\ref{thm:upperBoundJump}.
\end{proof}

We note that the GSEMO fulfills the assumptions of the Theorem with $q = \frac{1}{enS}$ and $q_k = \frac{1}{en^kS}$ as discussed in Section~\ref{sec:analysis_gsemo} (recall that $S$ is the size of the largest incomparable set for the considered benchmark, that is, the maximum size of the population in any generation).
For these values, Theorem~\ref{thm:meta-theorem} yields the same run-time bounds for the GSEMO as obtained in Section~\ref{sec:analysis_gsemo} (apart from a small discrepancy in the case of \mJump due to our slightly loose estimates).

As illustrated by the GSEMO, we note that $q$ can be estimated as $q\coloneqq s p$, where $s$ is a lower bound on the probability to select any specific solution from the population for mutation in a given iteration and $p$ is a lower bound on the probability that a solution $x$ selected for mutation produces any given Hamming neighbor $x$. For $q_k \coloneqq s p_k$, we have that $p_k$ needs to bound the probability of producing any given solution of Hamming distance $k$ to $x$.

In the following, we use the meta-theorem to obtain runtime bounds for the SEMO, \SMS, an \NSGAtwo variant, the \NSGAthree, and the \SPEA on the discussed benchmarks. 
We note that the meta-theorem gives runtime bounds in terms of iterations, which better reflects the proof structure, whereas we state our bounds for the individual MOEAs in terms of the number of fitness evaluations.
This allows for an easier comparison of the algorithms, which do not all evaluate the same number of solutions per iteration.% as each iteration of the (G)SEMO and \SMS just takes one fitness evaluation while the \NSGAthree and \SPEA perform multiple evaluations per iteration. 

\subsection{SEMO}

The first algorithm that we apply our meta-theorem to is the \emph{simple evolutionary multi-objective optimizer (SEMO)} proposed in~\cite{LaumannsTZ04}.
Apart from the mutation step, the SEMO and GSEMO are identical, see Algorithm~\ref{alg:gsemo}.
While the GSEMO uses bitwise mutation and thus flips bits independently, the SEMO uniformly at random selects one bit of the parent and flips that bit to create an offspring. With this local mutation operator, the SEMO cannot optimize jump-type benchmarks~\cite{DoerrZ21aaai}. For the unimodal benchmarks regarded in this work, we obtain very similar bounds as for the GSEMO. They are superior by a factor of $e$, an often observed difference between the local and the global mutation operator. 

More interestingly, exploiting the particular population dynamics of the SEMO on the \lotz problem, we can show that the population size often stays considerably below the size of a largest incomparable set. This allows to remove the, for the \mlotz problem very costly, size of the largest incomparable set from the runtime bound.

\begin{theorem}\label{thm:results_semo}
Let $m',n'\in \N$, $n=m'n'$, and  $m= 2m'$. 
%\merk{Ich würde mit $n$ und $m$ anfangen und dann darauf aufbauend $n'$ und $m'$ definieren. Bspw. ``Let $n, m \in \N$. Set $m' = \frac m2$ and $n' = \frac n{m'}$.''.} \simon{der Vorteil $m'$ und $n'$ zuerst zu definieren liegt darin, dass dann alles natürliche Zahlen sind, wir sollten ausschließen dass $m$ ungerade oder $n$ nicht durch $m$ teilbar ist}
Consider the SEMO optimizing a problem $\mathcal{B}$. Let $F$ denote the number of fitness evaluations until the population covers the complete Pareto front. 
    \begin{itemize}
    \itemsep1em 
        \item If $\mathcal{B} = \mOMM$ and 
        \[t\coloneqq \left(\tfrac{\ln(2)m'+2}{\ln(n)} + 16 \tfrac{{m'}^2+2m'}{n} + 2\right) n \ln(n) \SOMM ,\]
        then $F \le t+1$ with probability at least $(1-\frac{1}{n})^2 =1-\Theta(\frac{1}{n})$ and $E[F] \le (1-\frac{1}{n})^{-2} t+1$.
        \item If $\mathcal{B} =\ \mcocz$ and
        \[t \coloneqq\left(\tfrac{\ln(2)m'+2}{\ln(n)} + 16 \tfrac{{m'}^2+2m'}{n} + 4\right) n \ln(n) \SCOCZ,\]
        then $F \le t+1$ with probability at least $(1-\frac{1}{n})^3 =1-\Theta(\frac{1}{n})$ and $E[F] \le (1-\frac{1}{n})^{-3} t +1$.
        \item Let  $\mathcal{B} =\ \mlotz$ and 
        \begin{align*}
            r \coloneqq \max\left\{1, \tfrac{4 {m'}^2\ln\left(\tfrac{n}{m'}+1\right)+ 8 m'\ln(n)}{n}\right\} \tfrac{2n^2}{m'}.
        \end{align*}
        For all values of $m$, we have 
        $F \le r \SLOTZ+1$ with probability at least $1-\frac{1}{n}$ and $E[F] \le (1-\frac{1}{n})^{-1}r \SLOTZ+1$.

        For $m=o({n / \ln n)}$, we have $F \le r\cdot \left(O(n' \ln n)\right)^{m'}$ with probability at least $(1-\frac{1.01}{n})^2 =1-\Theta(\frac{1}{n})$.
        
        For $m=o(\sqrt{n / \ln n'})$, we have 
        \[E[F] \le r \cdot \left(O(n \ln n \right))^{m'} = \tfrac{n^2}{m} \left(O(n \ln n \right))^{m'}.\]
    \end{itemize}
\end{theorem}

\begin{proof}
Recall that the only difference between the SEMO and GSEMO is the mutation step. 
While the GSEMO flips each bit independently with probability $\frac{1}{n}$, the SEMO uniformly at random selects any bit and flips it.
Hence, the requirements (1) and (2) of Theorem~\ref{thm:meta-theorem} hold with $q \coloneqq \frac{1}{nS}$, where $S$ is the size of the largest incomparable set for the considered benchmark.
The initialization as well as each iteration takes one fitness evaluation.

The proof for the stronger bound for \mlotz requires several new and longer arguments, so we present it in separate lemmas in the following.
% When optimizing \mlotz, the special mutation operator used by the SEMO allows to bound the maximum size of the population much closer to the size of the Pareto front than before, which in turn significantly improves the runtime guarantee of the meta-theorem. 
% For improved readability, we prove this improved result in the lemmas following this theorem.
\end{proof}

% However, we do not obtain a bound for the \mJump benchmark as the SEMO is incapable of flipping $k > 1$ bits at the same time.
% Indeed, it is impossible for the SEMO to solve \mJump, as has been argued for the bi-objective case \cite{DoerrZ21aaai}.
% By the same reasoning, the SEMO cannot solve \mJump for any $m>2$.

We now analyze the runtime of the SEMO on \mlotz in the case that $m$ is not overly large (that is, we prove the last two assertions of the preceding theorem). The central observation is that the more restricted mutation operator of the SEMO allows us to better understand the maximum size of the population. In particular, we can show that the population size stays considerably below the maximum size of an incomparable set (we recall that for the \mlotz problem, this size is much larger than the size of the Pareto front). This allows to use our meta-theorem with a much higher probability for generating a particular neighbor of a solution in the population (essentially, we can replace the maximum size of an incomparable set $\SLOTZ$ by our estimate for the population size). 

In the following proofs, the distance of a solution from the Pareto front within a block will be a crucial notation. Let $x \in \{0,1\}^n$. Let $i \in [m']$. Recall that we agreed to denote the $i$-th block of $x$ by $x^i$. If $x$ is on the Pareto front, we have $\LO(x^i) + \TZ(x^i) = n'$. Hence we define the distance of $x$ from the Pareto front in the $i$-th block by
\[
d_i(x) = n'  - (\LO(x^i) + \TZ(x^i)).
\]

The key to our improved runtime bound for the SEMO on \mlotz is that we can bound the population size better than via the size of the largest incomparable set. The first step towards this is the observation that the population size remains one (at least) as long as no block has reached the Pareto front.

\begin{lemma}\label{lem:semo_pop_size_1}
  Consider a run of the SEMO on the \mlotz benchmark. If at some time $t$ the parent population $P_{t-1}$ of the SEMO consists of a single solution $x$ such that $d_i(x) > 0$ for all $i \in [m']$, then $|P_t|=1$ with probability one.
\end{lemma}

\begin{proof}
  Since $d_i(x) > 0$ for all $i$, a single bit flip applied to $x$ can change the objective value of at most one objective function~$f_j$. If this is a decrease, the offspring is discarded. If this is an increase, then the offspring replaces the parent. If no objective value is changed, the offspring also replaces the parent. Since in all cases exactly one of parent and offspring is taken into the next generation, we have $|P_t|+1$.
\end{proof}

In the following lemma, we show that while the population consists of a single individual~$x$, the value of $d_i(x)$ decreases at a similar rate in each block $i$. Hence at a suitable moment, the population consists of a single individual with all $d_i(x)$-values small. 

\begin{lemma}\label{lem:semo_progress_while_one_solution}
Let $m',n'\in \N$, $n=m'n'$, and  $m= 2m'$.
Consider a run of the SEMO on the \mLOTZ benchmark. Let $\frac{3}{n'} \le \delta < \frac{1}{2}$ and $t_0 \coloneqq (1-2\delta)\frac{n^2}{4m'}$. 
Then with probability at least
\mbox{$ 1-4m'\exp(-\frac{\delta^2(1-2\delta)n'}{24}+\frac{\delta^2}{6n})$}, the population after $\floor{t_0}$ iterations consists of a single individual~$x$ such that 
\begin{align*}
d_i(x) 
% &\le n' - (1-\tfrac{\delta}{2})^2(1-2\delta)n'+2-2\delta \\
% & \le \delta n' (3 - \tfrac 94 \delta + \tfrac 12 \delta^2) + 2 - 2 \delta
\le \delta  n'\left(3-\tfrac{9}{4}\delta+\tfrac{1}{2}\delta^2\right)+\tfrac{1}{2}\delta^2
\end{align*}
for all $i\in [m']$.
\end{lemma}

\begin{proof}
% We first estimate the probability of not sampling an individual $x$ with $d_i(x) = 0$ within the first $t$ iterations. 
% Under this condition, the population consists of a single individual at least up to iteration $t$ by Lemma~\ref{lem:lem:semo_pop_size_1}.
We have seen in the previous lemma that while all $d_i(x)$ values are positive, the population of the SEMO consists of a single individual. Since this is the regime this lemma takes place in, let us for the ease of language regard a fake version of the SEMO which is identical to the SEMO except when the offspring is incomparable to the parent; in this case the fake SEMO takes the offspring into the population and removes the parent. With this modification, we have ensured that the population $P_t$ always consists of a single individual $x(t)$. 

Let $i \in [m']$ be a block. To analyze how the individual $x(t)$ approaches the Pareto front, we analyze the growth of $Z_t^{(i)} := \LO(x^i(t)) + \TZ(x^i(t))$; we note that $Z_t^{(i)} + d_i(x(t)) = n'$. Let us use momentarily the shorthand $Z_t := Z_t^{(i)}$. From the definition of the algorithm and the \mlotz benchmark, we observe the following. Let $U$ be the sum of two independent geometric random variables (always taking values starting at~$1$) with success rate~$\frac 12$. Then for all $j \in [0..n'-2]$, the probability distribution of $Z_0$ agrees with $U-2$, that is, we have $\Pr[Z_0=j] = \Pr[U=j+2]$. Assume that for some $t$, we have $Z_t < n'$ (and hence also $Z_t \le n'-2$). Then the probability that $Z_{t+1} > Z_t$ is exactly $\frac 2n$, because there are exactly two bits such that flipping them increases $Z_t$, namely the first bit following the leading-ones segment and the last bit before the trailing-zeros segment of $x^i(t)$. In case such an improvement happens, the actual change $Z_{t+1}-Z_t$ is described essentially by a geometric distribution with success rate~$\frac 12$ because the subsequent bits are still independently and uniformly distributed (this is a classic argument from the analysis of the single-objective \leadingones problem, see, e.g., \cite{BottcherDN10,Sudholt13,Doerr19tcs}). Again, this is true only while $Z_{t+1} \le n'-2$. So formally, let $X_t$ be a Bernoulli random variable with success probability $\frac 2n$ and $U_t$ be a geometric random variable with success rate~$\frac 12$, then $Z_{t+1}$ is identically distributed to $Z_t + X_t U_t$ for all values that are at most $n'-2$.

Using an elementary induction, we obtain the following description of $Z_{\floor{t_0}}$, the progress after $\floor{t_0}$ iterations. Let $X_1, \dots, X_{\floor{t_0}}$ be Bernoulli random variables with success probability~$\frac 2n$. Let $Y_1, Y_2, \dots$ be an infinite sequence of geometric random variables with success rate $\frac 12$. Assume all these to be mutually independent.
Let $X = \sum_{t=1}^{\floor{t_0}} X_t$.
Let $Z = -2 + \sum_{t=1}^{X+2} Y_t$. Then $Z$ and $Z_{\floor{t_0}}$ are identically distributed on all values of at most $n'-2$, that is, $\Pr[Z=j]=\Pr[Z_{\floor{t_0}}=j]$ for all $j \in [0..n'-2]$.

% \merk{das versuche ich oben gerade umzuschreiben}
% For each block $i$, if the population consists of one individual $x$ such that $d_i(x)>0$ we say that an iteration is a \emph{success in block $i$} if it decreases $d_i(x)$. 
% Otherwise, we consider the iteration to be a success in block $i$ with probability $\frac{2}{n}$.
% In the first case, there are precisely two bits that, upon being flipped, decrease $d_i$ (the first $0$-bit and the last $1$-bit in block~$i$).
% Thus, the probability that any iteration is a success in block $i$ is precisely $\frac{2}{n}$.

% For a fixed block $i$ and $\ell \in [t]$, let $X_\ell$ be a random variable that is $1$ if iteration $\ell$ is a success in block $i$ and $0$ otherwise.
% Let $X = \sum_{\ell\in [t]} X_\ell$ and 

We now analyze the random variable~$Z$, starting with a probabilistic upper bound. To this aim, we first regard~$X$. By linearity of expectation, $E[X] = \sum_{\ell=1}^{\floor{t_0}} E[X_\ell] = \floor{t_0} \frac 2n,$ so $\frac 12 (1-2\delta) n' -\tfrac{2}{n} < E[X] \le \frac 12 (1-2\delta) n'$. Let $A$ be the event that $X\ge (1+\tfrac{\delta}{2})E[X] \eqqcolon r$. 
Using a multiplicative Chernoff bound, we have 
\begin{align*} 
 \Pr\left[A\right]
   \le  \exp\left(-\frac{(\delta/2)^2 E[X]}{3}\right)
   = \exp\left(-\frac{\delta^2}{12} E[X]\right).
\end{align*}
 
% Let $Z' \coloneqq Z+2$ and let $B$ be the event that $Z' \ge (1+\tfrac{\delta}{2})E[Z']$. 
Further, let $Z' = \sum_{t=1}^{\floor{r}+2} Y_t$ and let $B$ be the event that $Z' \ge (1+\tfrac{\delta}{2})E[Z']$.
Using a Chernoff bound for geometric random variables, e.g.,  \cite[Theorem~1.10.32~(a)]{Doerr20bookchapter} we obtain 
\begin{align*}
    \Pr[B] 
      \le \exp\left(-\frac{(\delta/2)^2 (\floor{r}+1)}{2(1+\delta/2)}\right)
    % &< \exp\left(-\frac{\delta^2 (r+1)}{10}\right)\\
    % \merk{maybe keep the $(1+\delta/2)$? But forget the $+1$?}
     % &\le \exp\left(-\frac{\delta^2 ((1+\delta/2) E[X]+1)}{10}\right)\\
   \le \exp\left(-\frac{\delta^2}{12} E[X]\right).
% \le \Pr[\sum_{t=1}^{X+2}Y_t \ge ... \mid X < ...]
% \le \Pr[Z'' \ge (1+\tfrac{\delta}{2})E[Z']] 
% \le .
\end{align*}
%Observe that $Y' < n'$ implies that for all sampled individuals~$x$ we have $d_i(x)>0$. 
We have $E[Z']=2\floor{r}+4 \le 2(1+\tfrac{\delta}{2})E[X]+4$.  Thus, if $B$ does not happen, then
\begin{align*}
    Z'  < (1+\tfrac{\delta}{2})E[Z'] 
    &\le (1+\tfrac{\delta}{2})^2 (1-2\delta)n'+2\delta+4 \\
   % & = 1-\delta-\tfrac{7\delta^2}{4}-\tfrac{\delta^3}{2}\\
   & < (1-\delta)n'+2\delta +4 < n' +2,
\end{align*}
% where we observe that 
% as well as
% $E[X] = \frac{2t}{n}$.
where the last estimate uses $2\delta+2 < 3$ and $\delta \ge \frac{3}{n'}$. Observe that if $A$ does not happen, then $Z \le Z' -2$. Thus if neither $A$ nor $B$ happens, we have $Z < n'$.

% For a fixed block $i$, we next bound $d_i(x)$ from above, where $x$ is the unique solution in $P_t$ under the assumption that neither $A$ nor $B$ happen.
% To this end, we bound $Y$ from below and observe $d_i(x) \le \max(0, n'-Y)$.
% We remark that $Y$ might overestimate the progress made due to an iteration marked as a success for $X$ when already $d_i(x)=0$ or that the geometric random variable would account for a decrease of $d_i(x)$ below $0$.
% In both cases, however, the iteration would result in $d_i(x)=0$ and thus the desired bound of 
% $d_i(x) \ge (1-\tfrac{\delta}{2})^2(1-\delta)n'+\delta$
% holds trivially.
% Thus, in the domain relevant for our purposes, $d_i(x)$ stochastically dominates $n'-Y$.

We now bound $Z$ from below. 
Let $A'$ be the event that $X \le (1-\tfrac{\delta}{2})E[X]\eqqcolon r'$ and 
$B'$ be the event that $Z'' := \sum_{t=1}^{\ceil{r'}+2} Y_t \le (1-\tfrac{\delta}{2})E[Z'']$.
Using again Chernoff bounds for independent binary and geometric random variables, e.g., \cite[Theorem 1.10.5 and 1.10.32~(b)]{Doerr20bookchapter}, we have
\begin{align*}
    \Pr[A'] & \le \exp\left(-\frac{\delta^2E[X]}{8}\right)
  \le \exp\left(-\frac{\delta^2}{12} E[X]\right),\\
    % &\le \exp\left(-\Theta(\delta^2 n')\right),\\
    \Pr[B']
    &\le 
    \exp\left(-\frac{\delta^2 (\ceil{r'}+2)}{8-16\delta/3}\right)\\
    &\le 
    \exp\left(-\frac{\delta^2(1-\delta/2)E[X]}{8-16\delta/3}\right)
   \le \exp\left(-\frac{\delta^2}{12} E[X]\right).
    % &\le \exp\left(-\Theta(\delta^2 n')\right).   
    % \Pr[D_i \le (1-\delta)n'] 
\end{align*}
If neither $A'$ nor $B'$ happen then 
\begin{align*}
    Z \ge Z''-2 &> (1-\tfrac{\delta}{2})E[Z''] -2 
    \ge (1-\tfrac{\delta}{2})(2(1-\tfrac{\delta}{2})E[X]+4) -2\\
    &= (1-\tfrac{\delta}{2})(2(1-\tfrac{\delta}{2})E[X])+ 2 -2\delta\\
    &> (1-\tfrac{\delta}{2})(2(1-\tfrac{\delta}{2})(\tfrac{1}{2}(1-2\delta)n'-\tfrac{2}{n}))+ 2 -2\delta\\
& = (1-\tfrac{\delta}{2})^2((1-2\delta)n'-\tfrac{4}{n})+ 2 -2\delta \\
 & > (1-\tfrac{\delta}{2})^2(1-2\delta)n' - \tfrac{1}{2}\delta^2
    % &..\\
    % &= (1-\tfrac{\delta}{2})^2(1-2\delta)n'.
    % &\ge (1-\tfrac{\delta}{2})^2(1-2\delta)n'+2-2\delta\\
    % &= \delta n' (3 - \tfrac 94 \delta + \tfrac 12 \delta^2) + 2 - 2 \delta.
\end{align*}
Thus, with probability at least
\begin{align*}
 1 - \Pr[A \lor B \lor A' \lor B'] 
 % & \ge 1 - (\Pr[A \lor B] + \Pr[A' \lor B']) \\  
 % & = 1 - (\Pr[A] + \Pr[B \mid \neg A]\Pr[\neg A] \\
 % & \qquad\quad + \Pr[A'] + \Pr[B' \mid \neg A']\Pr[\neg A']) \\  
 % & \ge 1 - (\Pr[A] + \Pr[B \mid \neg A] + \Pr[A'] + \Pr[B' \mid \neg A']) \\  
 % & \ge 1 - 4 \exp\left(-\frac{\delta^2(1-\frac{\delta}{2})E[X]}{12}\right),
 \ge  1-4\exp\left(-\frac{\delta^2}{12} E[X]\right)\\
 \ge  1-4\exp\left(-\frac{\delta^2(1-2\delta)n'}{24}+\frac{\delta^2}{6n}\right)
 % &=  1-\exp\left(-\Theta(\delta^2 n')\right)
\end{align*}
we have 
\[
% (1-\tfrac{\delta}{2})^2(1-2\delta)n' \le Z < n'.
(1-\tfrac{\delta}{2})^2(1-2\delta)n' - \tfrac{1}{2}\delta^2  \le Z \le  (1-\delta)n'+2\delta +2 < n'.
\]
As discussed above, the same probabilistic estimate holds for $Z_{t_0}^{(i)}$. Via a union bound over all blocks, we see that with probability at least $1-4m'\exp\left(-\frac{\delta^2(1-2\delta)n'}{24}+\frac{\delta^2}{6n}\right)$, all $Z_{t_0}^{(i)}$, $i \in [m']$, satisfy the above bounds.
Now  Lemma~\ref{lem:semo_pop_size_1} ensures that no incomparable solutions were ever generated up to time $t_0$, and thus our fake version of the SEMO behaved identical to the true SEMO. Hence the same estimates hold for the true SEMO. 
The claim follows as 
\[n' - \left((1-\tfrac{\delta}{2})^2(1-2\delta)n' - \tfrac{1}{2}\delta^2 \right)
= \delta n'\left(3-\tfrac{9}{4}\delta+\tfrac{1}{2}\delta^2\right)+\tfrac{1}{2}\delta^2.\]
% d_i = n'- Z
% d_i \le n' - ....\le 
\end{proof}

\begin{lemma}\label{lem:semo_bound_popsize}
Let $m',n'\in \N$, $n=m'n'$,  $m= 2m'$, and $0 \le D \le n'$.
Consider a run of the SEMO on the \mLOTZ benchmark such that at some time $t_0$, the parent population $P_{t_0-1}$ of the SEMO consists of a single solution $z$ such that $d_i(z) \le D$ for all $i\in [m']$. Then for all $t \ge t_0$, we have $|P_{t}| \le (n'+1+D^2)^{m'}$.
\end{lemma}

\begin{proof}
  We analyze which solutions can newly enter the population, more precisely, what is the objective value of such solutions. Since each mutation (which is a random bit flip) concerns only a single block, let us regard a single block $i \in [m']$ determining the two objectives $f_{2i-1}$ and $f_{2i}$. We first analyze new solutions $y$ with $d_i(y) > 0$. If the parent $x$ of the mutation operation satisfies $d(x)=0$, then no offspring $y$ of $x$ with $d(y)>0$ can enter the population (as it is dominated by $x$). If $d_i(x)>0$, then the offspring $y$ can only enter the population if $f_{2i-1}(y) \ge f_{2i-1}(x)$ and $f_{2i}(y) \ge f_{2i}(x)$. By induction, this implies $f_{2i-1}(y) \ge f_{2i-1}(z)$ and $f_{2i}(y) \ge f_{2i}(z)$.
  
  Consequently, any individual $y$ ever entering the population either satisfies $d(y) = 0$ or satisfies $f_{2i-1}(y) \ge f_{2i-1}(z)$ and $f_{2i}(y) \ge f_{2i}(z)$ and $d(y)>0$. The $y$ with $d_i(y)=0$ can have exactly $n' + 1$ different values for $(f_{2i-1}(y),f_{2i}(y))$. The possible values for $(f_{2i-1}(y),f_{2i}(y))$ of the $y$ with $f_{2i-1}(y) \ge f_{2i-1}(z)$ and $f_{2i}(y) \ge f_{2i}(z)$ and $d(y)>0$ are 
  \[
  \{(f_{2i-1}(z)+a,f_{2i}(z)+b) \mid a,b \in [0..D-2], a+b\le D-2\},
  \]
  which has cardinality at most $D^2$. Consequently, there are at most $n' + 1 + D^2$ different values that can occur as $(f_{2i-1}(y),f_{2i}(y))$, $y \in P_{t}$. 

  Regarding all blocks, we see that there are at most $(n' + 1 + D^2)^{m'}$ different objective values which can possibly occur in $P_{t}$, which is therefore an upper bound on the population size of $P_{t}$.  
\end{proof}

\begin{proof}[Proof of the asymptotic bounds for $\mlotz$ in Theorem~\ref{thm:results_semo}]
Let
\[\delta_1 = \sqrt{72 \cdot \frac{\ln(n)+\ln(4m')}{n'}}.\]
We immediately see that $\delta_1 \ge \frac{3}{n'}$.
Since $m' = o(n/\ln(n))$, we have $n'=\frac{n}{m'}=\omega(\ln(n)+\ln(m'))$, and thus $\delta_1 \le \frac{1}{3}$ for sufficiently large $n$. 
From this and again the definition of $\delta_1$, we estimate
\begin{align*}
p_1 &\coloneqq  4m'\exp\left(-\delta_1^2 (1-2\delta_1) \frac{n'}{24}+\frac{\delta^2}{6n}\right)\\
 &\le  \exp\left(\ln(4m')-\delta_1^2\frac{n'}{3\cdot 24}\right)\cdot\exp\left(\frac{1}{54n}\right) < \frac{1.01}{n},
\end{align*}
using $n\ge 2$ and $\exp(1/108) < 1.01$.
By observing that $\delta_1 = O(\sqrt{\ln(n)/n'}\,)$, we further have that

\begin{align*}
D_1  &\coloneqq\delta_1 n'\left(3-\tfrac{9}{4}\delta_1+\tfrac{1}{2}\delta_1^2\right)+\tfrac{1}{2}\delta_1^2
%
%not enough: \le n' - (\tfrac{5}{6})^2 (1-2\delta_1) n' 
= O(\delta_1 n') = O\left(\sqrt{\ln(n)n'}\right)
 \intertext{and thus}
S_1  &\coloneqq (n'+1+D_1^2)^{m'} = \left(O(\ln(n)n')\right)^{m'}.
\end{align*}
Since $\delta_1 \le \frac{1}{3}$, right from the definition of $D_1$ we obtain that $D_1 < n'$ for sufficiently large $n$.
Using Lemma~\ref{lem:semo_progress_while_one_solution}, the probability that for a run of the SEMO on \mlotz there is an iteration $t_0$ as specified in Lemma~\ref{lem:semo_bound_popsize} with $D\coloneqq D_1$ is at least $1-p_1 \ge 1-\frac{1.01}{n}$. By Lemma~\ref{lem:semo_bound_popsize}, the
population size will not grow beyond $S_1$ in the subsequent generations.
Thus, in each of these generations the probability to select any specific individual for mutation is at least $\frac{1}{S_1}$.
Applying Theorem~\ref{thm:meta-theorem} with $P_{t_0}$ as the initial population and setting $q\coloneqq \frac{1}{nS_1}$ yields the stated probabilistic runtime guarantee.

For the bound on the expected value of the runtime, different from the rest of the paper, we cannot use a restart argument since in this proof we build on the algorithm starting with a population of size one. Consequently, we have to work with different values for $\delta$ and $D$, ensuring that the failure probability for reaching a small population is so small that we can estimate the population size trivially in the failure case.

To this aim, let
\[\delta_2 = \sqrt{72\cdot \frac{2m'\ln(n'+1)+\ln(4m')}{n'}}.\]
Again we easily see that $\delta_2 \ge \frac{3}{n'}$. 
Since $m' = o(\sqrt{n/\ln(n'+1)})$, we have $n'=\tfrac{m'n}{m'^2}=\omega(\frac{m'n}{n/\ln(n'+1)}) = \omega(m'\ln(n'+1))$, and thus $\delta_2 \le \frac{1}{3}$ for sufficiently large $n$. Then
\begin{align*}
p_2 &\coloneqq  4m'\exp\left(-\delta_2^2 (1-2\delta_2) \frac{n'}{24}+\frac{\delta^2}{6n}\right)
 < 1.01  \exp\left(\ln(4m')-\delta_1^2\frac{n'}{3\cdot 24}\right) \\
 & = 1.01 \exp\left(-2m'\ln(n'+1)\right) = 1.01 (n'+1)^{-2m'}.
\end{align*}
By observing that $\delta_2 = O(\sqrt{m'\ln(n'+1)/n'})$, we have that
\begin{align*}
D_2 & \coloneqq \delta_2 n'\left(3-\tfrac{9}{4}\delta_2+\tfrac{1}{2}\delta_2^2\right)+\tfrac{1}{2}\delta_2^2
= O(\delta_2 n') = O\left(\sqrt{m'\ln(n'+1)n'}\right),
\end{align*}
and $\delta_2 \le \frac{1}{3}$ yields $D_2 < n'$ for sufficiently large $n$. Further,
\begin{align*}
S_2  \coloneqq (n'+1+D_2^2)^{m'} 
= \left(O(m'\ln(n'+1)n'\right)^{m'}
= \left(O(m'n'\ln(n)\right)^{m'}
\end{align*}

For a run of the SEMO on \mlotz, let $A$ be the event that the condition of Lemma~\ref{lem:semo_bound_popsize} is satisfied for $t_0 \coloneqq (1-2\delta_2)\frac{n^2}{4m'}$ and $D\coloneqq D_2$.
By Lemma~\ref{lem:semo_progress_while_one_solution} we have $\Pr[\neg A] \le p_2$.
We apply Theorem~\ref{thm:meta-theorem}, where we consider the initial population to be $P_{t_0}$.
If $A$ happens, we apply the theorem with $q\coloneqq \frac{1}{nS_2}$ as the population cannot grow beyond size $S_2$. 
Otherwise, we use the standard estimate $q\coloneqq \frac{1}{n\SLOTZ}$. The resulting bounds on the expected number of iterations are
\begin{align*}
    t_A &= (1-\tfrac{1}{n})^{-1}r S_2+1,\\
    t_{\neg A} &= (1-\tfrac{1}{n})^{-1}r \SLOTZ+1.
\end{align*}
We obtain 
\begin{align*}
    E[F] &\le t_0 + (1-p_2) t_A + p_2 t_{\neg A} \\
    %\le t_0 + t_A + p_2 t_{\neg A}\\
    &< t_0 + (1-\tfrac{1}{n})^{-1}r \left(S_2 + p_2\SLOTZ\right) +1
    = O\left(r S_2\right) 
% = \left(O(m'n'\ln(n'+1)\right)^{m'}
% < r\left(\left(9.03\cdot 600m'\ln(n'+1) n'+1\right)^{m'}+\tfrac{1}{8}\right)+1.
\end{align*}
as desired, noting that $p_2 \SLOTZ < 1.01 (n'+1)^{-2m'} (n'+1)^{2m'-1} < 1$.
% since for $\pow(a,b)=a^b$ we have
% \begin{align*}
%     p_2 \SLOTZ &\le \exp\left(-2m' \ln(n'+1)\right) (n'+1)^{2m'-1}\\
%     &= \pow\left(n'+1, 2m'-1 - \frac{2m'\ln(n'+1)}{\ln(n'+1)}\right) < 1.
%     % &= O(S_2). % \le \pow\left(n'+1, -1\right) < 1,
% \end{align*}
%\merk{missing: die kosten der initialisierung korrekt handeln. vermutlich: im theorem eine bemerkung machen, dass ohne diese, und egal für welche initialisierung, die schranken gelten mit $-1$.}
\end{proof}

\subsection{SMS-EMOA and $(\mu+1)$-SIBEA}
The \SMS works with a population of fixed size $\mu$. Similar to the (G)SEMO, it produces one offspring solution in each generation.
We consider the \SMS using uniform bitwise mutation just like employed by the GSEMO.
While the (G)SEMO only relies on the (strict) domination relation to select the surviving solutions for the next generation, the \SMS sorts solutions into fronts $F_1,\ldots, F_{i^*}$, where each front contains all pairwise not strictly dominating solutions that are not yet represented in an earlier front.
Then one element of $F_{i^*}$ is removed to reduce the population size back to $\mu$, namely one with smallest hypervolume contribution.
For some set $S$ and reference point $r$, the hypervolume of $S$ is 
$\text{HV}_r(S) = \mathcal{L}\left(\bigcup_{u\in S} \set{h\in\mathbb{R}^m \mid r\le h\le f(u)}\right),$
where $\mathcal{L}$ is the Lebesgue measure.
The hypervolume contribution of an individual of an individual $x\in F$ is  $\Delta_r(x,F) = \text{HV}_r(F) - \text{HV}_r(F\setminus\set{x})$. Since we only regard maximization problems with non-negative objective values, as common, we use the reference point $r=(-1,\dots,-1)$.
Algorithm~\ref{alg:smsemoa} states the \SMS in pseudocode.

\begin{algorithm2e}[t]%
\Input{%
objective function $f=(f_1,\ldots, f_m)$,\\
length of bitstrings $n\in\N$
}
    Let $P_0$ be a set of $\mu$ elements of $\{0,1\}^n$ chosen independently and uniformly at random\\
    \For{$t = 1, 2, \ldots$}{
        Select $x$ from $P_{t-1}$ uniformly at random and let $x'\coloneqq \mutate(x)$\\
        Divide $R_t\coloneqq P_{t-1}\cup \set{x'}$ into fronts $F_1,\ldots,F_{i^*}$, by fast-non-dominated-sort() \cite{DebPAM02}.\\
        Pick $z' \in \arg\min_{z\in F_{i^*}} \Delta_r(z,F_{i^*})$ uniformly at random and let $P_t \coloneqq R_t \setminus{z'}$
    }
\caption{SMS-EMOA with population size $\mu$ and bit-wise mutation $\mutate(\cdot)$.}
\label{alg:smsemoa}
\end{algorithm2e}%

{We remark that the simplified version of the \SIBEA as presented in \cite{BrockhoffFN08} is nearly identical to the \SMS as it uses the same selection and reproduction scheme. In fact, the only difference is their treatment of individuals that are not in the first non-dominated front for surviving into the next generation: while the \SMS uses the hypervolume contribution only as a tie breaker for individuals on the last non-dominated front, the \SIBEA skips the non-dominated sorting and drops a solution with minimum hypervolume contribution from the whole parent and offspring population.  
}

Our meta-theorem is applicable to the \SMS and the \SIBEA when $\mu$ is at least as large as the size $S$ of a largest set of incomparable solutions.
The resulting bounds are identical to the ones obtained for the GSEMO except that the factor $S$, which is the upper bound for the population size of the GSEMO used in our proofs, is replaced by the population size $\mu$ of the \SMS or \SIBEA.

\begin{theorem}\label{thm:boundsSMS}
   Consider the \SMS or the \SIBEA optimizing a problem $\mathcal{B}$ with a largest incomparable set of size $S$.
   Let this MOEA use population size $\mu \ge S$ and let $F$ denote the number of fitness evaluations until the population covers the complete Pareto front. Then, 
     \begin{itemize}
    \itemsep1em 
        \item If $\mathcal{B} = \mOMM$ and 
        \[t\coloneqq \left(\tfrac{\ln(2)m'+2}{\ln(n)} + 16 \tfrac{{m'}^2+2m'}{n} + 2\right) e  n \ln(n) \mu,\]
        then $F \le t+\mu$ with probability at least $(1-\frac{1}{n})^2 =1-\Theta(\frac{1}{n})$ and $E[F] \le (1-\frac{1}{n})^{-2} t+\mu$.
        \item If $\mathcal{B} =\ \mcocz$ and
        \[t \coloneqq \left(\tfrac{\ln(2)m'+2}{\ln(n)} + 16 \tfrac{{m'}^2+2m'}{n} + 4\right)e  n \ln(n) \mu,\]
        then $F \le t+\mu$ with probability at least $(1-\frac{1}{n})^3 =1-\Theta(\frac{1}{n})$ and $E[F] \le (1-\frac{1}{n})^{-3} t +\mu$.
        \item If $\mathcal{B} =\ \mlotz$ and 
        \[t \coloneqq \max\left\{1, \tfrac{4 {m'}^2\ln\left(n'+1\right)+ 8 m'\ln(n)}{n}\right\} 2 e \tfrac{n^2}{m'} \mu,\]
        then $F \le t+\mu$ with probability at least $1-\frac{1}{n}$ and $E[F] \le (1-\frac{1}{n})^{-1}t+\mu$.
     \item If $m' \ge 2$, $k\in[2...\frac{n'}{2}]$, $\mathcal{B} = \mJump$, and 
     \begin{align*}
         t \coloneqq \left(\ln(4)m'+\ln(n)+\ln(m')\right)\tfrac{13}{3} e n^k  \mu,
     \end{align*}
    then $F \le t+\mu$ with probability at least $(1-\frac{1}{n})^5 =1-\Theta(\frac{1}{n})$ and
    \begin{align*}
      E[F] \le \left(1-\tfrac{1}{n}\right)^{-2} \left(\ln(4)m'+2\ln(m')\right)\tfrac{13}{3} e  n^k \mu +\mu.
    \end{align*}
    \end{itemize}
\end{theorem}
\begin{proof}
    The \SMS fulfills the monotonicity property required by Theorem~\ref{thm:meta-theorem} \cite[Lemma~4]{ZhengD24}.
    The \SIBEA is monotonic for similar reasons: if a solution $x$ is non-dominated, either there is a solution $x'$ with the same objective value or $x$ has positive hypervolume contribution.  
    In the first case, at least one of the solutions $x, x'$ survives into the next iteration. 
    In the latter case, $x$ survives as there are at most $S \le \mu$ solutions with positive hypervolume contribution, hence the combined parent and offspring population contains at least one solution with hypervolume contribution zero.
    
    Each iteration, the \SMS or \SIBEA selects a solution from the population for mutation uniformly at random, so any given solution $x$ in the population is selected with probability at least $\frac{1}{\mu}$.
    As we assume that the same mutation operator as for the GSEMO is employed, we set $q\coloneqq \frac{1}{en\mu}$ and $q_k \coloneq \frac{1}{en^k\mu}$.
    Each iteration takes one fitness evaluation and the initialization takes $\mu$ fitness evaluations.
{The bounds for \mJump follow by using the tighter estimates for $w_1$ and $w_2$ in the proof of Theorem~\ref{thm:meta-theorem} and combining them like we did for the GSEMO in the proof of Theorem~\ref{thm:upperBoundJump}, the only difference being that $\SOJZJ$ is replaced by $\mu$.}
\end{proof}

\subsection{NSGA-II, balanced NSGA-II, and NSGA-III}
We brief{}ly describe the central characteristics of the \NSGAtwo and \NSGAthree and refer to the original works~\cite{DebPAM02,DebJ14} or the theoretical works~\cite{ZhengD23aij,WiethegerD23} for more details.
Both algorithms work with a population of fixed size $\mu$. 
Each iteration, every individual produces an offspring, here we assume by uniform bitwise mutation. 
The combined population of size $2\mu$ is sorted into ranks, each containing all solutions  only dominated by those with lower rank.
The new generation of size $\mu$ is selected by prioritizing lower rank solutions.
As a tiebreaker, the \NSGAtwo employs the crowding distance, while the \NSGAthree uses reference points in the solution space.

Our meta-theorem does not apply to the \NSGAtwo (except for two objective) as this algorithm does not satisfy the monotonicity property: even with a large population, the \NSGAtwo can lose non-dominated objective values as observed in \cite{ZhengD24many}. %This matches that \NSGAtwo indeed has difficulties on benchmarks with more than 2 objectives as discussed in the introduction.
In contrast, the \NSGAthree preserves non-dominated solutions for reasonable choices of the population size and reference points~\cite{OprisDNS24}.

Very recently, however, a second tiebreaker (after the crowding distance) was proposed that gives the monotonicity property to the \NSGAtwo~\cite{DoerrIK25}. For the resulting \emph{balanced \NSGAtwo}, the following result follows from~\cite[Lemma~4]{DoerrIK25} and\cite[equation~(1)]{ZhengD24many}.

\begin{lemma}\label{lem:nsga2_variant_monotonic}
  Consider some multi-objective optimization problem such that the largest incomparable set of solutions has size at most~$S$ and such that, for all $i \in [1..m]$, the number of different objective values in the $i$-th objective is $\nu_i$. Consider solving this problem via the balanced \NSGAtwo~\cite{DoerrIK25}. If the population size satisfies $\mu \ge S + 2 \sum_{i=1}^m \nu_i$, then the monotonicity property is fulfilled.
\end{lemma}

\begin{proof}
   By equation~(1) of \cite{ZhengD24many}, $U := 2 \sum_{i=1}^m \nu_i$ is an upper bound on the number of individuals having positive crowding distance in any incomparable set of solutions. Since $\mu \ge S+U$, Lemma~4 of~\cite{DoerrIK25} shows the ``survival property'' that a non-dominated objective value in the combined parent and offspring population survives into the next generation, that is, if the first front of the non-dominated sorting of $R_t$ contains a solution $x$, then the next generation $P_{t+1}$ contains a solution $y$ such that $f(y) = f(x)$. This survival property easily implies our monotonicity property.
\end{proof}

With the monotonicity property of the balanced \NSGAtwo and the \NSGAthree, our meta-theorem gives the following result.

\begin{theorem}
\label{thm:boundsNSGA}
   Consider the balanced \NSGAtwo of \cite{DoerrIK25} or the \NSGAthree optimizing a problem $\mathcal{B}$ with a largest incomparable set of size $S$. In case of the balanced \NSGAtwo, denote by $\nu_i$ the number of different objective values in the $i$-th objective and assume that $\mu \ge S+2\sum_{i=1}^m \nu_i$. In the case of the \NSGAthree, assume that all objective values are non-negative and are at most $\fmax$, that $\mu \ge S$, and that the \NSGAthree employs a set of reference points $\mathcal{R}_p$ as defined in \cite{OprisDNS24} with $p \ge 2m^{3/2}\fmax$.
     Let $F$ denote the number of fitness evaluations until the population covers the complete Pareto front.  
    \begin{itemize}
    \itemsep1em 
        \item If $\mathcal{B} = \mOMM$ and 
        \[t\coloneqq \left(\tfrac{\ln(2)m'+2}{\ln(n)} + 16 \tfrac{{m'}^2+2m'}{n} + 2\right) e  n \ln(n) \mu,\]
        then $F \le t+\mu$ with probability at least $(1-\frac{1}{n})^2 =1-\Theta(\frac{1}{n})$ and $E[F] \le (1-\frac{1}{n})^{-2} t+\mu$.
        \item If $\mathcal{B} =\ \mcocz$ and
        \[t \coloneqq \left(\tfrac{\ln(2)m'+2}{\ln(n)} + 16 \tfrac{{m'}^2+2m'}{n} + 4\right)e  n \ln(n) \mu,\]
        then $F \le t+\mu$ with probability at least $(1-\frac{1}{n})^3 =1-\Theta(\frac{1}{n})$ and $E[F] \le (1-\frac{1}{n})^{-3} t +\mu$.
        \item If $\mathcal{B} =\ \mlotz$ and 
        \[t \coloneqq \max\left\{1, \tfrac{4 {m'}^2\ln\left(n'+1\right)+ 8 m'\ln(n)}{n}\right\} 2 e \tfrac{n^2}{m'} \mu\]
        then $F \le t+\mu$ with probability at least $1-\frac{1}{n}$ and $E[F] \le (1-\frac{1}{n})^{-1}t+\mu$.
     \item If $m' \ge 2$, $k\in[2...\frac{n'}{2}]$, $\mathcal{B} = \mJump$, and 
     \begin{align*}
         t \coloneqq \left(\ln(4)m'+\ln(n)+\ln(m')\right)\tfrac{13}{3} e n^k  \mu,
     \end{align*}
    then $F \le t+\mu$ with probability at least $(1-\frac{1}{n})^5 =1-\Theta(\frac{1}{n})$ and
    \begin{align*}
      E[F] \le \left(1-\tfrac{1}{n}\right)^{-2} \left(\ln(4)m'+2\ln(m')\right)\tfrac{13}{3} e  n^k \mu +\mu.
    \end{align*}
    \end{itemize}
\end{theorem}
\begin{proof}
    Both MOEAs with the given choices of population size and reference points fulfill the monotonicity property, see Lemma~\ref{lem:nsga2_variant_monotonic} and~\cite{OprisDNS24}.
    Each iteration, all solutions in the current population are mutated by bitwise-mutation like for the GSEMO, so the second property holds with $q \coloneqq \frac{1}{en}$ and we have $q_k \coloneqq \frac{1}{en^k}$. 
    The initialization as well as each iteration takes $\mu$ fitness evaluations.
    For the bounds for \mJump we once more used the tighter estimates for $w_1$ and $w_2$ in the proof of Theorem~\ref{thm:meta-theorem}.
\end{proof}

%\merk{B: Vielleicht lieber irgendwo zum gro\ss en Gesamtvergleich ausholen? S: Lieber hier lassen.}
Even though the \NSGAthree and the \NSGAtwo variant require fewer iterations than the (G)SEMO and \SMS, these algorithms only perform one fitness evaluation per iteration. As here we mutate $\mu$ solutions (and hence perform $\mu$ fitness evaluations) in each iteration, the runtime bounds on these variants of the \NSGAtwo with respect to the number of fitness evaluations indeed exactly match the bounds on the \SMS (apart from the additional $\mu-1$ fitness evaluations in the initialization).

\subsection{SPEA2}
\label{sec:spea}

Yet another MOEA massively used in practice is the \emph{Strength Pareto Evolutionary Algorithm 2} (SPEA2) proposed by Zitzler, Laumanns, and Thiele~\cite{ZitzlerLT01} (only a technical report, but cited more than 10,000 times according to Google scholar). 

Due to the different language and setup, this algorithm appears very different from the remaining algorithms considered in this work. However, a closer look at the algorithm details reveals that it is actually very similar to the \NSGAtwo, in particular, when noting that what is called an archive in~\cite{ZitzlerLT01} is rather a parent population. Also, the uncommon order of the steps in the main loop, first a selection for survival and then the creation of offspring, can be rearranged to the more common order of first creating offspring and then selecting the next parent population from the existing parents and offspring. 

We give such a reformulation of the SPEA2 in Algorithm~\ref{alg:spea}. We let the reader confirm that, apart from the presentation, it is identical to the SPEA2 defined in~\cite{ZitzlerLT01}. 
The only difference is that we specify what happens in the case that the offspring population size~$\lambda$ is smaller than the parent population size~$\mu$, in this case, the original SPEA2 appears not well-defined. 
For the convenience of the reader, we describe in Table~\ref{tab:conversion} how to convert the language and notation of~\cite{ZitzlerLT01} into our language and notation, which follow the ones traditionally used for many other classic MOEAs.
%We assume that the $\mutate()$ operation performs uniform bitwise mutation.

We do not fully specify here how the next parent population is selected from the combined parent and offspring population. For our purposes, it suffices to know that the first priority is given to non-dominated solutions. If these are too many, a complicated hierarchy of distances $\sigma^k$ is used as secondary criterion. This has, in particular, the property that solutions for which another solution with identical objective value exists are removed first. If there are less than $\mu$ non-dominated solutions, the next parent population is filled up with dominated solutions. Since our proofs only rely on non-dominated solutions, the particular details of this step are not relevant here.

\begin{algorithm2e}[ht]%
\Input{%
objective function $f=(f_1,\ldots, f_m)$,\\
length of bitstrings $n\in\N$
% Objective functions $f_1,\ldots, f_m : \{0,1\}^n \to \R$; $f:= (f_1, \dots, f_m)$
}
  $P_0 := \emptyset$\\
    \For{$t = 1, 2, \ldots$}{
      \If{$t=1$}{
       Let $Q_1$ be a set of $\max\{\lambda,\mu\}$ elements of $\{0,1\}^n$ chosen independently and uniformly at random
      }  
      \Else{
        $Q_t := \emptyset$\\
        \For{$i=1$ \KwTo $\lambda$}{
          Select $x \in P_{t-1}$ uniformly at random\\
          % Generate $y$ from $x$ by flipping each bit independently with probability $1/n$\\
          $Q_t := Q_t \cup \{\mutate(x)\}$\\
        }
      }
      Let $P_{t}$ be the set of non-dominated solutions in $P_{t-1} \cup Q_t$\\
      \If{$|P_{t}| > \mu$}{
         Remove individuals in $P_{t}$ using the $\sigma^k$ distances until $|P_{t}|=\mu$    
      }
      \ElseIf{$|P_{t}| < \mu$}{
         Add individuals from $(P_{t-1} \cup Q_t) \setminus P_{t}$ to $P_{t}$, based on a strength-based quality indicator, until $|P_{t}|=\mu$
      }
    }
\caption{An equivalent formulation of the SPEA2 with parent population size~$\mu$, offspring population size~$\lambda$, and bit-wise mutation $\mutate(\cdot)$.}
\label{alg:spea}
\end{algorithm2e}%

\begin{table}[]
   \centering
    \renewcommand{\arraystretch}{1.5}
    \begin{tabular}{l|l}
       \toprule
       Original SPEA2~\cite{ZitzlerLT01}  & Our SPEA2 formulation \\\midrule
       archive $\overline P_{t}$  & parent population $P_t$\\\hline
       archive size $\overline N$ & parent population size $\mu$\\\hline
       population $P_t$ & offspring population $Q_t$\\\hline
       population size $N$ & offspring population size $\lambda$\\\bottomrule
    \end{tabular}
    \vspace{2mm}
\caption{Comparison of the notation in~\cite{ZitzlerLT01} and our formulation of the SPEA2.}
    \label{tab:conversion}
\end{table}

% \begin{algorithm}[t!]
% 	\caption{The original SPEA2 as defined in~\cite{RenBLQ24}. \merk{THIS IS JUST FOR US TO CHECK THAT THE TWO ALGORITHMS ARE REALLY CORRECT. AT SOME TIME, THIS SHOULD BE REMOVED.}}
% 	\label{alg:SPEA2}
% 	\textbf{Input}: objective functions $f_1,f_2...,f_m$, population size $\mu$, archive size $\bar{\mu}$  \\
% 	%	\textbf{Parameter}: Optional list of parameters\\
% 	\textbf{Output}:  $\bar{\mu}$ solutions from $\{0,1\}^n$
% 	\begin{algorithmic}[1] 
% 		\STATE $P\leftarrow \mu$ solutions uniformly and randomly selected from $\{0,\! 1\}^{\!n}$ with replacement;
%             \STATE $A = \emptyset$;
% 		\WHILE{criterion is not met}
%             \STATE $A' \leftarrow $ non-dominated solutions in $ P \cup A$;
% 		\IF{$|A'| > \bar{\mu}$}
%   		\STATE reduce $A'$ by means of the truncation operator
%             \ELSIF{$|A'| < \bar{\mu}$}
%   		\STATE fill $A'$ with dominated individuals in $P$ and $A$
% 		\ENDIF
%             \STATE $A  \leftarrow A'$;
%             \STATE let $P'=\emptyset$, $i=0$;
%             \WHILE{$i<\mu$}
%             \STATE select a solution from $A$ uniformly at random;
% 		\STATE generate $x'$ by ﬂipping each bit of $x$ independently with probability $1/n$;
% 		\STATE $P'= P'\cup \{x'\}$, $i= i+1$
%             \ENDWHILE
%             \STATE $P\leftarrow P'$
% 		\ENDWHILE
% 		\RETURN $A$
% 	\end{algorithmic}
% \end{algorithm}

A simple consequence of Lemma~4 in~\cite{RenBLQ24} is that the SPEA2 fulfills the monotonicity property when $\mu$ is large enough.

\begin{lemma}\label{lem:speamono}
  Consider some multi-objective optimization problem such that the largest incomparable set of solutions has size at most~$S$. Consider solving this problem with the SPEA2 with parent population size~$\mu$. If $\mu \ge S$, then the monotonicity property is satisfied.
\end{lemma}

\begin{proof}
  Consider some solution $x$ generated in some iteration~$t$ or present in the parent population of this iteration. If $x$ is non-dominated in the combined parent and offspring population, then by Lemma~4 of \cite{RenBLQ24}, some solution $y$ with $f(y)=f(x)$ survives into the next parent population. If $x$ is dominated by some solution $z$, then a solution $y$ with $f(y) = f(z)$ survives, and this is a solution dominating $x$. This shows the monotonicity property.
\end{proof}

With the monotonicity property, we now easily derive from Theorem~\ref{thm:meta-theorem} the following runtime guarantees.

\begin{theorem}\label{thm:boundsSPEA2}
   Consider the \SPEA optimizing a problem $\mathcal{B}$ with a largest incomparable set of size at most~$S$.
   Let the \SPEA use a parent population of size~$\mu \ge S$ and an offspring population of size~$\lambda \ge 1$.
     Let $F$ denote the number of fitness evaluations until the population covers the complete Pareto front.
    \begin{itemize}
    \itemsep1em 
      \item If $\mathcal{B} = \mOMM$ and
        \[t \coloneqq \left(\tfrac{\ln(2)m'+2}{\ln(n)} + 16 \tfrac{{m'}^2+2m'}{n} + 2\right) (\lambda + en\mu) \ln(n),\]
        then $F \le t+\max\set{\lambda, \mu}$ with probability at least $(1-\frac{1}{n})^2 =1-\Theta(\frac{1}{n})$ and $E[F] \le (1-\frac{1}{n})^{-2} t+\max\set{\lambda, \mu}$.
        \item If $\mathcal{B} =\ \mcocz$ and
        \[t \coloneqq \left(\tfrac{\ln(2)m'+2}{\ln(n)} + 16 \tfrac{{m'}^2+2m'}{n} + 4\right)(\lambda + en\mu) \ln(n),\]
        then $F \le t+\max\set{\lambda, \mu}$ with probability at least $(1-\frac{1}{n})^3 =1-\Theta(\frac{1}{n})$ and $E[F] \le (1-\frac{1}{n})^{-3} t +\max\set{\lambda, \mu}$.
        \item If $\mathcal{B} =\ \mlotz$ and 
        \begin{align*}
            t \coloneqq \max\Big\{1, \tfrac{4 {m'}^2\ln\left(n'+1\right)+ 8 m'\ln(n)}{n}\Big\}  \, 2 (\lambda + en\mu) n'
        \end{align*}
        then $F \le t+\max\set{\lambda, \mu}$ with probability at least $1-\frac{1}{n}$ and $E[F] \le (1-\frac{1}{n})^{-1}t+\max\set{\lambda, \mu}$.
     \item If $m' \ge 2$, $k \in [2..\frac{n'}{2}]$, $\mathcal{B} = \mJump$, and  
     \begin{align*}
         t \coloneqq  \big(14m'
              +  \tfrac{16{m'}^2}{n}\ln\left(n'-2k+3\right) 
            +6\ln(n)  +2\big)\\
            \cdot(\lambda + en\mu)
              + \tfrac{7}{3}\big(\ln(4)m'+\ln(n)+ \ln(m')\big)(\lambda + en^k\mu),
     \end{align*}
    then $F \le t+\max\set{\lambda, \mu}$ with probability at least $(1-\frac{1}{n})^5 =1-\Theta(\frac{1}{n})$ and $E[F] \le (1-\tfrac{1}{n})^{-2} \Big(14m' +  \tfrac{16{m'}^2}{n}\ln\left(n'-2k+3\right) 
            +6\ln(n) +2\Big)\\\cdot(\lambda+en\mu)
             + (1-\tfrac{1}{m'})^{-1}\tfrac{7}{3}(\ln(4)m'+2\ln(m')) (\lambda+en^k\mu) + \max\set{\lambda,\mu}$.
    \end{itemize}
\end{theorem}
\begin{proof}
  We first show that the assumptions of Theorem~\ref{thm:meta-theorem} are satisfied. For the monotonicity property, this was just seen in Lemma~\ref{lem:speamono}. 
  The probability that in one iteration we select a particular parent individual and mutate it into a particular solution in Hamming distance one is at least $1 - (1 - \frac{1}{en\mu})^\lambda \ge 
  1 - \frac{en\mu}{en\mu+\lambda} = \frac{\lambda}{\lambda+en\mu}\eqqcolon q$, where the first inequality stems from Lemma~2 in~\cite{AntipovD21algo}, which again builds heavily on arguments from~\cite{RoweS14}.
  Similarly, the probability to generate a particular solution in Hamming distance~$k$ is at least $q_k \eqqcolon \frac{\lambda}{\lambda+en^k\mu}$.
  % This is at least $\frac 12$, when $\lambda \ge \mu$, and at least $\frac{\lambda}{2\mu}$, when $\lambda \le \mu$. Hence the probability to pick a particular individual as parent is at least $\frac 12 \min\{1, \frac \lambda \mu\} =: s$. Since we use bitwise mutation with mutation rate~$\frac 1n$, the probability to generate a particular Hamming neighbor is $\frac 1n (1-\frac 1n)^{n-1} \ge \frac 1 {en} =:p_1$, and the probability to generate a particular individual in distance~$k$ is $n^{-k}(1-\frac 1n)^{n-k} \ge \frac 1 {en^k} =: p_k$, as for all other algorithms using this mutation operator.  
  The initialization takes $\max\set{\lambda,\mu}$ fitness evaluations and each iteration performs $\lambda$ fitness evaluations.
 For our bounds we note that $\frac{\lambda}{q} = \lambda + en\mu$ as well as $\frac{\lambda}{q_k} = \lambda + en^k\mu$.
  % For our bounds we note that $\lambda \cdot (\frac 12\min\set{1,\frac{\lambda}{\mu}})^{-1} = 2\max\set{\lambda, \mu}$.
\end{proof}

We note that in the case $\lambda=1$, our estimates above immediately give the minimally stronger estimates $q = \frac 1 {en\mu}$ and $q_k = \frac 1 {en^k \mu}$. With these, the bounds proven in the theorem become identical to the ones previously proven for the \SMS.

Our main result in this subsection considerably improves and extends the results in the first runtime analysis for the SPEA2 in~\cite{RenBLQ24}. For this comparison, note that their work denotes the parent population size (in our language) by $\bar\mu$ and the offspring population size by $\mu$. To reduce the confusion, we shall state all results in our notation. We also note that all runtime bounds in~\cite{RenBLQ24} mistakenly depend on the offspring population size instead of the parent population size (that is, any $\mu$ should be replaced by $\bar\mu$ in~\cite{RenBLQ24}); this mistake is corrected in the arxiv version of that paper. 

% With this correction and in our notation, \cite{RenBLQ24} give a runtime guarantee of $O(\mu n \min\{m \log(n), n\})$ for the \mOMM  benchmark, whereas we showed $O(\mu n \max\{1,\frac{\lambda}{\mu n}\} \max\{m, \log(n), \frac{m^2 \log n}{n}\})$. We have not explanation for the fact that the bound in~\cite{RenBLQ24} does not depend on $\lambda$ as it is quite clear that a large value of $\lambda$ must have a negative impact on the runtime. 
With this correction and in our notation, \cite{RenBLQ24} gives a runtime guarantee of $O(\mu n^2)$ for the \mLOTZ  benchmark, whereas ours is $O(\mu n \max\{1,\frac{\lambda}{\mu n}\} \cdot\max\{\frac{n}{m}, m \log(\frac{n}{m})\})$. The missing dependence on $\lambda$ in the bound in~\cite{RenBLQ24} is explained by the, somewhat hidden, assumption that $\lambda = c\mu$ for some $c \in [\frac 1\mu, O(1)]$, see the paragraph before Section 3.1 in~\cite{RenBLQ24}. In this regime, our bound is superior to the one of \cite{RenBLQ24} for all $m$ that are $\omega(1)$ and $o(n)$. 

In the interest of brevity, we spare a detailed comparison for the other benchmarks, but note that there again our results extend those of~\cite{RenBLQ24} to arbitrary offspring sizes and that they are asymptotically stronger for many parameter settings.

\section{Heavy-Tailed Mutation}
In the previous sections we exclusively considered algorithms that employed uniform bitwise mutation as a mutation operator or, in case of the SEMO, flip a single bit chosen uniformly at random.
In order to demonstrate the generalizability of the meta-theorem (Theorem~\ref{thm:meta-theorem}) and to extend our runtime bounds to another well-established mutation operator, in this section we consider heavy-tailed mutation.

The heavy-tailed mutation operator, also called \emph{fast mutation}, was proposed in~\cite{DoerrLMN17} and has since found numerous successful applications~\cite{WuQT18,FriedrichGQW18,AntipovBD22}. It also has inspired heavy-tailed random choices of other parameters such as offspring population sizes, selection pressures, and others~\cite{DangELQ22,AntipovBD24,DoerrKV24}. In multi-objective optimization, speed-ups by a factor of $k^{\Omega(k)}$ from heavy-tailed mutation were shown for the GSEMO~\cite{ZhengD23ecj} and \NSGAtwo~\cite{DoerrQ23tec} optimizing the bi-objective \jump problem, and for the \SMS optimizing the many-objective jump problem~\cite{ZhengD24}.

The heavy-tailed mutation operator flips each bit independently, just like uniform bitwise mutation. 
Instead of a fixed probability such as~$\frac{1}{n}$ to flip a bit, it uses the probability~$\frac{\alpha}{n}$, where $\alpha$ is sampled from a power-law distribution with parameter~$\beta$, independently for each application of the mutation operator. 
Formally, for this power-law distribution with parameter $\beta$ we have $\Pr[\alpha = i] = (\sum^{n/2}_{j=1} j^{-\beta})^{-1} i^{-\beta}$ for all $i\in [1\ldots,\frac{n}{2}]$.
%\simon{Ich glaube ihr habt da einen Typo in \cite{ZhengD24}, ihr schreibt $k^{-\beta}$ statt $j^{-\beta}$. B: Ja, stimmt. Ist schon korrigiert in }
Crucially, as proven in~\cite{DoerrLMN17}, for all $i\in [n/2]$,  this operator has a chance of $\Omega(i^{-\beta})$ to flip exactly $i$ bits.

\begin{theorem}
    \label{thm:heavyTailedMutation}
 Consider a modified version of the GSEMO, \SMS, \SIBEA, \NSGAthree, or \SPEA that uses heavy-tailed mutation with constant parameter $\beta$ as the mutation operator. 
 Then, apart from constant factors, the upper bounds of Theorems~\ref{thm:upperBoundOMM}, \ref{thm:upperBoundCOCZ}, and ~\ref{thm:upperBoundJump} also hold for the modified GSEMO. Likewise, the asymptotic orders of magnitude of the upper bounds of Theorems~\ref{thm:boundsSMS}-\ref{thm:boundsSPEA2} hold for the modified \SMS, \SIBEA, \NSGAthree, and \SPEA, respectively. 
 For each of the discussed algorithms with heavy-tailed mutation, the guarantees on $\mJump$  improve by a factor of $k^{\Omega(k)}$.
\end{theorem}

\begin{proof}
Changing the mutation operator does not alter the monotonicity property nor the probability to select a specific individual for mutation.
%Hence, using the same value for $s$ as the original algorithms, the modified algorithms satisfy the first two requirements of Theorem~\ref{thm:meta-theorem}.
Since the heavy-tailed mutation operator with constant probability flips one bit (uniformly chosen), we can apply the meta-theorem with asymptotically the same value as before, and from this obtain all results  for \mOMM, \mcocz, and \mlotz.
% The probability of the heavy-tailed mutation operator with parameter $\beta$ performing an $a$-bit flip for $a\in \N$ is in $\Omega(a^{-\beta})$ \cite[Lemma~19]{ZhengD24arxiv} \simon{actually this is also in the AAAI version (in the prose, not in a Lemma). So cite that?}. 
% For the third property, as there are $n$ distinct bitstrings with Hamming distance one to a given solution $x$, mutating $x$ to any specific one of them has a probability in $\Omega(1^{-\beta})\cdot \frac{1}{n} = \Omega(\frac{1}{n})$.
% Hence, applying Theorem~\ref{thm:meta-theorem} using $p_1 = \frac{c}{n}$ for a suitable constant $c$ yields the same asymptotic bounds as when using bitwise mutation for \mOMM, \mcocz, and \mlotz.

Concerning \mJump, mutating a solution $x$ into any specific bitstring of Hamming distance $k$ to $x$ has a probability of 
\[\Omega\left(\frac{k^{-\beta}}{\binom{n}{k}}\right) = \Omega(e^{-k} k^{k+0.5-\beta} n^{-k}) = k^{\Omega(k)}n^{-k},\]
where we estimate 
$\binom{n}{k} \le \frac{n^k}{k!} \le \frac{n^k}{\sqrt{2\pi} k^{k+0.5}e^{-k}}$ using  $k! \ge \sqrt{2\pi} k^{k+0.5}e^{-k}$.
This allows for a value of $q_k$ that is by a factor of $k^{\Omega(k)}$ larger than in the proof so far, giving then the claimed speed-up for \mJump. 
\end{proof}

\section{Conclusion}

In this work, we revisited the problem of proving performance guarantees for MOEAs dealing with more than two objectives. In the first major progress after the initial work on this question twenty years ago~\cite{LaumannsTZ04}, we analyzed the runtime of many classic algorithms on four classic benchmarks. Most of our runtime guarantees are linear in the size of the Pareto front (apart from small factors polynomial in $n$ and $m$), in contrast to the previous bounds, which were all quadratic in the size of the Pareto front. Our results thus suggest that MOEAs cope much better with many-objective problems than known. In fact, our work indicates that the performance loss observed when increasing number of objectives in experiments is caused by the increasing size of the Pareto front rather than by particular algorithmic difficulties of many-objective optimization. 

An obvious next step in this research direction would be to advance from synthetic benchmarks to classic combinatorial optimization problems, e.g., the multi-objective minimum spanning tree problem regarded so far only for two objectives~\cite{Neumann07,DoNNS23,CerfDHKW23}. A technical challenge we could not overcome is to determine the runtimes for the many-objective LOTZ problem, where we do not have the obvious lower bound of the Pareto front size. We note, though, that lower bounds for LOTZ are already very difficult in the bi-objective setting~\cite{DoerrKV13}. %\mtodo{S: Mich irritiert die Formulierung ein bisschen, wir haben ja eine Schranke gezeigt (wie bei allen anderen auch), der Unterschied ist nur dass der Abstand zwischen unterer und oberer Schranke größer ist.  }

\section*{Acknowledgments}
Simon Wietheger acknowledges the support by the Austrian Science Fund (FWF, Grant DOI 10.55776/Y1329).
Benjamin Doerr's research benefited from the support of the FMJH program PGMO.
This work has profited from many deep discussions at the Dagstuhl Seminars \href{https://www.dagstuhl.de/23361}{23361 \emph{Multiobjective Optimization on a Budget}} and \href{https://www.dagstuhl.de/24271}{24271 \emph{Theory of Randomized Optimization Heuristics}}.

% \end{linenumbers}

%% If you have bib database file and want bibtex to generate the
%% bibitems, please use
%%
  \bibliographystyle{elsarticle-num} 
\bibliography{alles_ea_master,ich_master,rest}

%% else use the following coding to input the bibitems directly in the
%% TeX file.

%% Refer following link for more details about bibliography and citations.
%% https://en.wikibooks.org/wiki/LaTeX/Bibliography_Management

% \begin{thebibliography}{00}

% %% For numbered reference style
% %% \bibitem{label}
% %% Text of bibliographic item

% \bibitem{lamport94}
%   Leslie Lamport,
%   \textit{\LaTeX: a document preparation system},
%   Addison Wesley, Massachusetts,
%   2nd edition,
%   1994.

% \end{thebibliography}
} %end sloppy
\end{document}